\documentclass[11pt]{article}
\usepackage{graphicx}
\usepackage{caption}
\usepackage{subcaption}
\usepackage{natbib}
\usepackage[algoruled,vlined,nofillcomment]{algorithm2e}
\usepackage{hyperref}

\usepackage{authblk}

\usepackage{amsmath}
\usepackage{amsthm}
\usepackage{amsfonts}
\usepackage{amssymb}

\usepackage{fullpage}

\newcommand{\R}{\mathbb{R}}
\newcommand{\N}{\mathbb{N}}
\newcommand{\E}{\mathbb{E}}
\renewcommand{\P}{\mathbb{P}}
\newcommand{\one}{\mathbb{I}}

\newcommand{\norm}[1]{\| #1 \|}
\newcommand{\bignorm}[1]{\left\| #1 \right\|}

\DeclareMathOperator{\argmin}{argmin}

\DeclareMathOperator{\Tr}{Tr}

\newcommand{\cX}{\mathcal{X}}
\newcommand{\cY}{\mathcal{Y}}

\newcommand{\cN}{\mathcal{N}}

\newcommand{\GS}{\mathrm{GS}}
\newcommand{\LS}{\mathrm{LS}}

\newcommand{\states}{\mathcal{S}}

\newcommand{\vtime}[2]{i_{#1,#2}}
\newcommand{\withstate}[2]{#1_{#2}}
\newcommand{\Rmax}{R_{\mathrm{max}}}
\newcommand{\Fmax}{F_{\mathrm{max}}}
\newcommand{\Xz}{X^\circ}
\newcommand{\sig}[1]{\langle #1 \rangle}
\newcommand{\bthe}{\bar{\theta}}

\newtheorem{theorem}{Theorem}

\newtheorem{lemma}[theorem]{Lemma}
\newtheorem{corollary}[theorem]{Corollary}
\theoremstyle{remark}

\title{Differentially Private Policy Evaluation\footnote{Corresponding e-mail: \url{b.deballepigem@lancaster.ac.uk}}}
\date{}
\author[1]{Borja Balle}
\affil[1]{Department of Mathematics and Statistics, Lancaster University, UK}
\author[2]{Maziar Gomrokchi}
\author[2]{Doina Precup}
\affil[2]{School of Computer Science, McGill University, Canada}

\begin{document} 

\maketitle

\begin{abstract} 
We present the first differentially private algorithms for reinforcement learning, which apply to the task of evaluating a fixed policy. We establish two approaches for achieving differential privacy, provide a theoretical analysis of the privacy and utility of the two algorithms, and show promising results on simple empirical examples.
\end{abstract} 

\section{Introduction}

Learning how to make decisions under uncertainty is becoming paramount in many practical applications, such as medical treatment design, energy management, adaptive user interfaces, recommender systems etc. Reinforcement learning~\citep{sutton1998reinforcement} provides a variety of algorithms capable of handling such tasks.  However, in many practical applications, aside from obtaining good predictive performance, one might also require that the data used to learn the predictor be kept confidential.  This is especially true in medical applications, where patient confidentiality is very important, and in other applications which are user-centric (such as recommender systems). \emph {Differential privacy} (DP)~\citep{dwork2006differential} is a very active research area, originating from cryptography, but which has now been embraced by the machine learning community. DP is a formal model of privacy used to design mechanisms that reduce the amount of information leaked by the result of queries to a database containing sensitive information about multiple users~\citep{dwork2006differential}. Many supervised learning algorithms have differentially private versions, including logistic regression \citep{chaudhuri2009privacy,chaudhuri2011differentially}, support vector machines \citep{chaudhuri2011differentially,rubinstein2012learning,jain2013differentially}, and the lasso \citep{thakurta2013differentially}. However, differential privacy for reinforcement learning tasks has not been tackled yet, except for the simpler case of bandit problems~\citep{smith2013,mishra2015,tossou:aaai2016}.

In this paper, we tackle differential privacy for reinforcement learning algorithms for the full Markov Decision Process (MDP) setting.  We develop differentially private algorithms for the problem of policy evaluation, in which a given way of behaving has to be evaluated quantitatively. We start with the batch, first-visit Monte Carlo approach to policy evaluation, which is well understood and closest to regression algorithms, and provide two differentially private versions, which come with formal privacy proofs as well as guarantees on the quality of the solution obtained.  Both algorithms work by injecting Gaussian noise into the parameters vector for the value functions, but they differ in the definition of the noise amount. Our privacy analysis techniques are related to previous output perturbation for empirical risk minimization (ERM), but there are some domain specific challenges that need to be addressed. Our utility analysis identifies parameters of the MDP that control how easy it is to maintain privacy in each case. The theoretical utility analysis, as well as some illustrative experiments, show that the accuracy of the private algorithms does not suffer (compared to usual Monte Carlo) when the data set is large.

The rest of the paper is organized as follows. In Sec.~\ref{sec:background} we provide background notation and results on differential privacy and Monte Carlo methods for policy evaluation. Sec.~\ref{sec:dpalgs} presents our proposed algorithms. The privacy analysis and the utility analysis are  outlined in Sec.~\ref{sec:priv} and Sec.~\ref{sec:util} respectively. Detailed proofs for both of these sections are given in the Supplementary Material. In Sec.~\ref{sec:exp} we provide empirical illustrations of the scaling behaviour of the proposal algorithms, using synthetic MDPs, which try to mimic characteristics of real applications. Finally, we conclude in Sec.~\ref{sec:conclu} with a discussion of related work and avenues for future work.

\section{Background}\label{sec:background}

In this  section we provide background on differential privacy and  policy evaluation from Monte Carlo estimates.

\subsection{Differential Privacy}\label{sec:backDP}

DP takes a user-centric approach, by providing privacy guarantees based on the difference of the outputs of a learning algorithm trained on two databases differing in a single user. The central goal is to bound the loss in privacy that a user can suffer when the result of an analysis on a database with her data is made public. This can incentivize users to participate in studies using sensitive data, e.g.\ mining of medical records. In the context of machine learning, differentially private algorithms are useful because they allow learning models in such a way that their parameters do not reveal information about the training data~\citep{mcsherry2007mechanism}. For example, one can think of using historical medical records to learn prognostic and diagnostic models which can then be shared between multiple health service providers without compromising the privacy of the patients whose data was used to train the model.

To formalize the above discussion, let $\cX$ be an \emph{input space} and $\cY$ an \emph{output space}. Suppose $A$ is a randomized algorithm that takes as input a tuple $X = (x_1, \ldots, x_m)$ of elements from $\cX$ for some $m \geq 1$ and outputs a (random) element $A(X)$ of $\cY$. We interpret $X \in \cX^m$ as a dataset containing data from $m$ individuals and define its \emph{neighbouring} datasets as those that differ from $X$ in their last\footnote{Formally, we should define neighbouring datasets as those which differ in one element, not necessarily the last. But we are implicitly assuming here that the order of the elements in $X$ does not affect the distribution of $A(X)$, so we can assume without loss of generality that the difference between neighbouring datasets is always in the last element.} element: $X' = (x_1,\ldots, x_{m-1}, x_m')$ with $x_m \neq x_m'$. We denote this (symmetric) relation by $X \simeq X'$. $A$ is \emph{$(\varepsilon,\delta)$-differentially private} for some $\varepsilon, \delta > 0$ if for every $m \geq 1$, every pair of datasets $X, X' \in \cX^m$, $X \simeq X'$, and every measurable set $\Omega \subseteq \cY$ we have
\begin{equation}
\P[A(X) \in \Omega] \leq e^{\varepsilon} \P[A(X') \in \Omega] + \delta \enspace.
\end{equation}
This definition means that the distribution over possible outputs of $A$ on inputs $X$ and $X'$ is very similar, so revealing this output leaks almost no information on whether $x_m$ or $x_m'$ was in the dataset.

A simple way to design a DP algorithm for a given function $f : \cX^m \to \cY$ is the \emph{output perturbation} mechanism, which releases $A(X) = f(X) + \eta$, where $\eta$ is noise sampled from a properly calibrated distribution. For real outputs $\cY = \R^d$, the Laplace (resp.\ Gaussian) mechanism~(see e.g.\ \citet{dwork2014algorithmic}) samples each component of the noise $\eta = (\eta_1,\ldots,\eta_d)$ i.i.d.\ from a Laplace (resp.\ Gaussian) distribution with standard deviation $O(\GS_1(f)/\varepsilon)$ (resp.\ $O(\GS_2(f) \ln(1/\delta) /\varepsilon)$), where $\GS_p(f)$ is the \emph{global sensitivity} of $f$ given by
\begin{equation*}
\GS_p(f) = \sup_{X,X' \in \cX^m, X \simeq X'} \norm{f(X) - f(X')}_p \enspace.
\end{equation*}
Calibrating noise to the global sensitivity is a worst-case approach that requires taking the supremum over all possible pairs of neighbouring datasets, and in general does not account for the fact that in some datasets privacy can be achieved with substantially smaller perturbations. In fact, for many applications (like the one we consider in this paper) the global sensitivity is too large to provide useful mechanisms. Ideally one would like to add perturbations proportional to the potential changes around the input dataset $X$, as measured, for example by the \emph{local sensitivity} $\LS_p(f,X) = \sup_{X' \simeq X} \norm{f(X) - f(X')}_p$. \citet{nissim2007smooth} showed that approaches based on $\LS_p$  do not lead to differentially private algorithms, and then proposed an alternative framework for DP mechanisms with data-dependent perturbations based on the idea of \emph{smoothed sensitivity}. This is the approach we use in this paper; see Section~\ref{sec:priv} for further details.

\subsection{Policy Evaluation}\label{sec:backPE}

Policy evaluation is the problem of obtaining (an approximation to) the value function of a Markov reward process defined by an MDP $M$ and a policy $\pi$ \citep{sutton1998reinforcement,szepesvari2010algorithms}. In many cases of interest $M$ is unknown but we have access to trajectories containing state transitions and immediate rewards sampled from $\pi$. When the state space of $M$ is relatively small, tabular methods that represent the value of each state can be used individually. However, in problems with large (or even continuous) state spaces, parametric representations for the value function are typically needed in order to defeat the curse of dimensionality and exploit the fact that similar states will have similar values. In this paper we focus on policy evaluation with linear function approximation in the batch case, where we have access to a set of trajectories sampled from the policy of interest.

Let $M$ be an MDP over a finite state space $\states$ with $N = |\states|$ and $\pi$ a policy on $M$. Given an initial state $s_0 \in \states$, the interaction of $\pi$ with $M$ is described by a sequence $((s_t,a_t,r_t))_{t \geq 0}$ of state--action--reward triplets. Suppose $0 < \gamma < 1$ is the discount factor of $M$. The value function $V^\pi : \states \to \R$ of $\pi$ assigns to each state the expected discounted cumulative reward obtained by a trajectory following policy $\pi$ from that state:
\begin{equation}
V^\pi(s) = \E_{M, \pi} \left[\, \textstyle\sum\nolimits_{t \geq 0} \gamma^t r_t \; \middle| \; s_0 = s \, \right] \enspace.
\end{equation}
The value function can be considered a vector $V^\pi \in \R^{\states}$. We make the usual assumption that any reward $r$ generated by $M$ is bounded: $0 \leq r \leq \Rmax$, so $0 \leq V^\pi(s) \leq \Rmax / (1 - \gamma)$ for all $s \in \states$.

Let $\Phi \in \R^{\states \times d}$ be a feature representation that associates each state $s \in \states$ to a $d$-dimensional feature vector $\phi_s^\top = \Phi(s,:) \in \R^d$. The goal is to find a parameter vector $\theta \in \R^d$ such that $\hat{V}^\pi = \Phi \theta$ is a good approximation to $V^\pi$. To do so, we assume that we have access to a collection $X = (x_1,\ldots, x_m)$ of finite trajectories sampled from $M$ by $\pi$, where each $x_i$ is a sequence of states, actions and rewards.

We will use a Monte Carlo approach, in which the returns of the trajectories in $X$ are used as regression targets to fit the parameters in $\hat{V}^\pi$ via a least squares approach \citep{sutton1998reinforcement}. In particular, we consider first-visit Monte Carlo estimates obtained as follows. Suppose $x = ((s_1,a_1,r_1),\ldots,(s_T,a_T,r_T))$ is a trajectory that visits $s$ and $\vtime{x}{s}$ is the time of the first visit to $s$; that is, $s_{\vtime{x}{s}} = s$, and $s_t \neq s$ for all $t < \vtime{x}{s}$. The return collected from this first visit is given by
\begin{equation*}
F_{x,s} =
\sum_{t = \vtime{x}{s}}^T r_t \gamma^{t - \vtime{x}{s}} =
\sum_{t = 0}^{T - \vtime{x}{s}} r_{t + \vtime{x}{s}} \gamma^t \enspace,
\end{equation*}
and provides an unbiased estimate of $V^\pi(s)$. For convenience, when state $s$ is not visited by trajectory $x$ we assume $F_{x,s} = 0$.

Given the returns from all first visits corresponding to a dataset $X$ with $m$ trajectories, we can find a parameter vector for the estimator $\hat{V}^\pi$ by solving the optimization problem $\argmin_{\theta} J_X(\theta)$, where
\begin{equation}\label{eq:JX}
J_X(\theta) = \frac{1}{m} \sum_{i = 1}^m \sum_{s \in \states_{x_i}} \rho_s (F_{x_i,s} - \phi_s^\top \theta)^2 \enspace,
\end{equation}
and $\states_x$ is the set of states visited by trajectory $x$. The regression weights $0 \leq \rho_s \leq 1$ are given as an input to the problem and capture the user's believe that some states are more relevant than others. It is obvious that $J_X(\theta)$ is a convex function of $\theta$. However, in general it is not strongly convex and therefore the optimum of $\argmin_{\theta} J_X(\theta)$ is not necessarily unique. On the other hand, it is known that differential privacy is tightly related to certain notions of stability \citep{thakurta2013differentially}, and optimization problems with non-unique solutions generally pose a problem to stability. In order to avoid this problem, the private policy evaluation algorithms that we propose in Section~\ref{sec:dpalgs} are based on optimizing slightly modified versions of $J_X(\theta)$ which promote stability in their solutions. Note that the notions of stability  related to DP  are for worst-case situations: that is, they need to hold for every possible pair of neighbouring input dataset $X \simeq X'$, regardless of any generative model assumed for the trajectories in those datasets. In particular, these stability considerations are not directly related to the variance of the estimates in $\hat{V}^\pi$.

We end this section with a discussion of the main obstruction to stability, i.e.\ the cases where $\argmin_{\theta} J_X(\theta)$ fails to have a unique solution. Given a dataset $X$ with $m$ trajectories we define a vector $F_X \in \R^{\states}$ containing the average first visit returns from all trajectories in $X$ that visit a particular state. In particular, if $\withstate{X}{s}$ represents the multiset of trajectories from $X$ that visit state $s$ at some point, then we have
\begin{equation}\label{eq:FX}
F_X(s) = F_{X,s} = \frac{1}{|\withstate{X}{s}|} \sum_{x \in \withstate{X}{s}} F_{x,s} \enspace.
\end{equation}
If $s$ is not visited by any trajectory in $X$ we set $F_{X,s} = 0$. To simplify notation, let $F_X \in \R^{\states}$ be the vector collecting all these estimates. We also define a diagonal matrix $\Gamma_X \in \R^{\states \times \states}$ with entries given by the product of the regression weight on each state and the fraction of trajectories in $X$ visiting that state: $\Gamma_X(s,s) = \rho_s |\withstate{X}{s}| / m$. Solving for $\theta$ in $\nabla_{\theta} J_X(\theta) = 0$, it is easy to see that any optimal $\theta_X \in \argmin_{\theta} J_X(\theta)$ must satisfy
\begin{equation}\label{eq:optgammax}
\Phi^\top \Gamma_X \Phi \theta_X = \Phi^\top \Gamma_X F_X \enspace.
\end{equation}
Thus, this optimization has a unique solution if and only if the matrix $\Phi^\top \Gamma_X \Phi$ is invertible. Since it is easy to find neighbouring datasets $X \simeq X'$ where at most one of $\Phi^\top \Gamma_X \Phi$ and $\Phi^\top \Gamma_{X'} \Phi$ is invertible, optimizing $J_X(\theta)$ directly poses a problem to the design differentially private policy evaluation algorithms with small perturbations. Next we present two DP algorithm based on stable policy evaluation algorithms.

\section{Private First-Visit Monte Carlo Algorithms}\label{sec:dpalgs}

In this section we give the details of two differentially private policy evaluation algorithms based on first-visit Monte Carlo estimates. Each of these algorithms corresponds to a different stable version of the minimization $\argmin_\theta J_X(\theta)$ described in previous section. A formal privacy analysis of these algorithms is given in Section~\ref{sec:priv}. Bounds showing how the privacy requirement affects the utility of the value estimates are presented in Section~\ref{sec:util}.

\subsection{Algorithm DP-LSW}\label{sec:alg1}

One way to make the optimization $\argmin_\theta J_X(\theta)$ more stable to changes in the dataset $X$ is to consider a similar least-squares optimization where the optimization weights do not change with $X$, and guarantee that the optimization problem is always strongly convex. Thus, we consider a new objective function given in terms of a new set of positive regression weights $w_s > 0$. Let $\Gamma \in \R^{\states \times \states}$ be a diagonal matrix with $\Gamma(s,s) = w_s$. We define the objective function as:
\begin{equation}
J_X^w(\theta) = \sum_{s \in \states} w_s (F_{X,s} - \phi_s^\top \theta)^2 = \norm{F_X - \Phi \theta}_{2,\Gamma}^2 \enspace,
\end{equation}
where  $\norm{v}_{2,\Gamma}^2 = \norm{\Gamma^{1/2} v}_2^2 = v^\top \Gamma v$ is  the weighted $\mathrm{L}_2$ norm. To see the relation between the optimizations over $J_X$ and $J_X^w$, note that equating the gradient of $J_X^w(\theta)$ to $0$ we see that a minimum $\theta_X^w \in \argmin_\theta J_X^w(\theta)$ must satisfy 
\begin{equation}\label{eq:optgamma}
\Phi^\top \Gamma \Phi \theta_X^w = \Phi^\top \Gamma F_X \enspace.
\end{equation}
Thus, the optimization problem is well-posed whenever $\Phi^\top \Gamma \Phi$ is invertible, which henceforth will be our working assumption. Note that this is a mild assumption, since it is satisfied by choosing a feature matrix $\Phi$ with full column rank. Under this assumption we have:
\begin{equation}\label{eq:thetaXw}
\theta_X^w = \left(\Phi^\top \Gamma \Phi\right)^{-1} \Phi^\top \Gamma F_X =
\left(\Gamma^{1/2} \Phi\right)^{\dagger} \Gamma^{1/2} F_X \enspace,
\end{equation}
where $M^{\dagger}$ denotes the Moore--Penrose pseudo-inverse. The difference between optimizing $J_X(\theta)$ or $J_X^w(\theta)$ is reflected in the differences between \eqref{eq:optgammax} and \eqref{eq:optgamma}. In particular, if the trajectories in $X$ are i.i.d.\ and $p_s$ denotes the probability that state $s$ is visited by a trajectory in $X$, then taking $w_s = \E_X[\rho_s |X_s| / m] = \rho_s p_s$ yields a loss function $J_X^w(\theta)$ that captures the effect of each state $s$ in $J_X(\theta)$ in the asymptotic regime $m \to \infty$. However, we note that knowledge of these visit probabilities is not required for running our algorithm or for  our analysis.

Our first DP algorithm for policy evaluation applies a carefully calibrated output perturbation mechanism to the solution $\theta_X^w$ of $\argmin_\theta J_X^w(\theta)$. We call this algorithm DP-LSW, and its full pseudo-code is given in Algorithm~\ref{alg:1}. It receives as input the dataset $X$, the regression weights $w$, the feature representation $\Phi$, and the MDP parameters $\Rmax$ and $\gamma$. Additionally, the algorithm is parametrized by the privacy parameters $\varepsilon$ and $\delta$. Its output is the result of adding a random vector $\eta$ drawn from a multivariate Gaussian distribution $\cN(0,\sigma_X^2 I)$ to the parameter vector $\theta_X^w$. In order to compute the variance of $\eta$ the algorithm needs to solve the discrete optimization problem $\psi_X^w = \max_{0 \leq k \leq K_X} e^{-k \beta} \varphi_X^w(k)$, where $K_X = \max_{s \in \states} |X_s|$, $\beta$ is a parameter computed in the algorithm, and $\varphi_X^w(k)$ is given by the following expression:
\begin{equation}\label{eq:psiw}
\varphi_X^w(k) = \sum_{s \in \states} \frac{w_s}{\max\{|X_s| - k,1\}^2}\enspace.
\end{equation}
Note that $\psi_X^w$ can be computed in time $O(K_X N)$.
%Thus, the total running time of the algorithm is $O(BLA BLA)$, where BLA BLA.

\begin{algorithm}[h]
\caption{DP-LSW}\label{alg:1}
\KwIn{$X$, $\Phi$, $\gamma$, $\Rmax$, $w$, $\varepsilon$, $\delta$}
\KwOut{$\hat{\theta}_X^w$}
\BlankLine
Compute $\theta_X^w$ \tcp*{cf.\ \eqref{eq:thetaXw}}
Let $\alpha \leftarrow \frac{5\sqrt{2 \ln(2/\delta)}}{\varepsilon}$ and $\beta \leftarrow \frac{\varepsilon}{4 (d + \ln(2/\delta))}$\;
Let $\psi_X^w \leftarrow \max_{0 \leq k \leq K_X} e^{-k \beta} \varphi_X^w(k)$ \tcp*{cf.\ \eqref{eq:psiw}}
Let $\sigma_X \leftarrow \frac{\alpha \Rmax \norm{(\Gamma^{1/2} \Phi)^{\dagger}}}{1-\gamma} \sqrt{\psi_X^w}$\;
%Let $\sigma_X \leftarrow \frac{\alpha \Rmax}{(1-\gamma) \sigma_{\min}(\Gamma^{1/2} \Phi)} \sqrt{\psi_X^w}$\;
Sample a $d$-dimensional vector $\eta \sim \mathcal{N}(0,\sigma_X^2 I)$\;
Return $\hat{\theta}_X^w = \theta_X^w + \eta$\;
\end{algorithm}

The variance of the noise in DP-LSW is proportional to the upper bound $\Rmax/(1-\gamma)$ on the return from any state. This bound might be excessively pessimistic in some applications, leading to unnecessary large perturbation of the solution $\theta_X^w$. Fortunately, it is possible to replace the term $\Rmax/(1-\gamma)$ with any smaller upper bound $\Fmax$ on the returns generated by the target MDP on any state. In practice this leads to more useful algorithms, but it is important to keep in mind that for the privacy guarantees to remain unaffected, one needs to assume that $\Fmax$ is a publicly known quantity (i.e.\ it is not based on an estimate made from private data). These same considerations apply to the algorithm in the next section.

\subsection{Algorithm DP-LSL}\label{sec:alg2}

The second DP algorithm for policy evaluation we propose is also  an output perturbation mechanism. It differs from DP-LSW in they way stability of the unperturbed solutions is promoted. In this case, we choose to optimize a regularized version of $J_X(\theta)$. In particular, we consider the objective function $J_X^\lambda(\theta)$ obtained by adding a ridge penalty to the least-squares loss from \eqref{eq:JX}:
\begin{equation}
J_X^\lambda(\theta) = J_X(\theta) + \frac{\lambda}{2m} \norm{\theta}_2^2 \enspace,
\end{equation}
where $\lambda > 0$ is a regularization parameter. The introduction of the ridge penalty makes the objective function $J_X^\lambda(\theta)$ strongly convex, and thus ensures the existence of a unique solution $\theta_X^\lambda = \argmin_\theta J_X^\lambda(\theta)$, which can be obtained in closed-form as:
\begin{equation}\label{eq:thetaXl}
\theta_X^\lambda = \left(\Phi^\top \Gamma_X \Phi + \frac{\lambda}{2m} I\right)^{-1} \Phi^{\top} \Gamma_X F_X \enspace.
\end{equation} 
Here $\Gamma_X$ is defined as in Section~\ref{sec:backPE}.

We call DP-LSL the algorithm obtained by applying an output perturbation mechanism to the minimizer of $J_X^\lambda(\theta)$; the full pseudo-code is given in Algorithm~\ref{alg:2}. It receives as input  the privacy parameters $\varepsilon$ and $\delta$, a dataset of trajectories $X$, the regression weights $\rho$, the feature representation $\Phi$, a regularization parameter $\lambda > \norm{\Phi}^2 \norm{\rho}_\infty$, and the MDP parameters $\Rmax$ and $\gamma$. After computing the solution $\theta_X^\lambda$ to $\argmin_\theta J_X^\lambda(\theta)$, the algorithm outputs $\hat{\theta}_X^\lambda = \theta_X^\lambda + \eta$, where $\eta$ is a $d$-dimensional noise vector drawn from  $\cN(0,\sigma_X^2 I)$. The variance of $\eta$ is obtained by solving a discrete optimization problem (different from the one in DP-LSW). Let $c_\lambda = \norm{\Phi} \norm{\rho}_\infty / \sqrt{2 \lambda}$ and for $k \geq 0$, define $\varphi_X^\lambda(k)$ as:
%\varphi_X^w(k) = \frac{\norm{\Phi} \norm{\rho}_\infty \sqrt{\sum_{s \in \states} \rho_s \min\{|X_s| + k, m\}}}{\sqrt{2 \lambda}}  + \norm{\rho}_2 \enspace.
\begin{equation}\label{eq:psil}
%\varphi_X^\lambda(k) =
\left(
c_\lambda
%\frac{\norm{\Phi} \norm{\rho}_\infty}{\sqrt{2 \lambda}}
\sqrt{\sum_{s} \rho_s \min\{|X_s| + k, m\}}
+ \norm{\rho}_2\right)^2 \enspace.
\end{equation}
%where $c_\lambda = \norm{\Phi} \norm{\rho}_\infty / \sqrt{2 \lambda}$.
Then DP-LSL computes $\psi_X^\lambda = \max_{0 \leq k \leq m} e^{-k \beta} \varphi_X^\lambda(k)$, which can be done in time $O(m N)$.

\begin{algorithm}[h]
\caption{DP-LSL}\label{alg:2}
\KwIn{$X$, $\Phi$, $\gamma$, $\Rmax$, $\rho$, $\lambda$, $\varepsilon$, $\delta$}
\KwOut{$\hat{\theta}_X^\lambda$}
\BlankLine
Compute $\theta_X^\lambda$ \tcp*{cf.\ \eqref{eq:thetaXl}}
%(cf. Equation \ref{eq:thetaXl})\;
Let $\alpha \leftarrow \frac{5 \sqrt{2 \ln(2/\delta)}}{\varepsilon}$ and $\beta \leftarrow \frac{\varepsilon}{4 (d + \ln(2/\delta))}$\;
Let $\psi_X^\lambda \leftarrow \max_{0 \leq k \leq m} e^{-k \beta} \varphi_X^\lambda(k)$ \tcp*{cf.\ \eqref{eq:psil}}
%(cf. Equation \ref{eq:psil})\;
Let $\sigma_X \leftarrow \frac{2 \alpha \Rmax \norm{\Phi}}{(1-\gamma)(\lambda - \norm{\Phi}^2 \norm{\rho}_\infty)} \sqrt{\psi_X^\lambda}$\;
Sample a $d$-dimensional vector $\eta \sim \mathcal{N}(0,\sigma_X^2 I)$\;
Return $\hat{\theta}_X^\lambda = \theta_X^\lambda + \eta$\;
\end{algorithm}

\section{Privacy Analysis}\label{sec:priv}

This section provides a formal privacy analysis for  DP-LSW and DP-LSL and shows that both algorithms are $(\varepsilon,\delta)$-differentially private. We use the smooth sensitivity framework of \citep{nissim2007smooth,Nissim2011JounalVersion}, which provides tools for the design of DP mechanisms with data-dependent output perturbations. We rely on the following lemma, which provides sufficient conditions for calibrating Gaussian output perturbation mechanisms with variance proportional to smooth upper bounds of the local sensitivity.

\begin{lemma}[\citet{Nissim2011JounalVersion}]\label{lem:Gaussian}
Let $A$ be an algorithm that on input $X$ computes a vector $\mu_X \in \R^d$ deterministically and then outputs $Z_X \sim \mathcal{N}(\mu_X, \sigma_X^2 I)$, where $\sigma_X^2$ is a variance that depends on $X$. Let $\alpha = \alpha(\varepsilon, \delta) =
5 \sqrt{2 \ln(2/\delta)} / \varepsilon$
%15 \sqrt{2 \ln(4/\delta)} / \varepsilon$
and $\beta = \beta(\varepsilon, \delta, d) =
\varepsilon / (4 d + 4 \ln(2/\delta))$.
Suppose $\varepsilon$ and $\delta$ are such that the following are satisfied for every pair of neighbouring datasets $X \simeq X'$: (a) $\sigma_X \geq \alpha \norm{\mu_X - \mu_{X'}}_2$, and (b) $|\ln(\sigma_X^2) - \ln(\sigma_{X'}^2)| \leq \beta$.
Then $A$ is $(\varepsilon,\delta)$-differentially private.
\end{lemma}

Condition (a) says we need variance  at least proportional to the local sensitivity $\LS_2(f,X)$. Condition (b) asks that the variance does not change too fast between neighbouring datasets, by imposing the constraint $\sigma_X^2 / \sigma_{X'}^2 \leq e^{\beta}$. This is precisely the spirit of the smoothed sensitivity principle: calibrate the noise to a smooth upper bound of the local sensitivity. We acknowledge Lemma~\ref{lem:Gaussian} is only available in pre-print form, and thus provide an elementary proof in Appendix~\ref{ap:sec:smthgauss} for completeness. The remaining proofs from this section are presented Appendices~\ref{ap:sec:priv1} and~\ref{ap:sec:priv2}.

\subsection{Privacy Analysis of DP-LSW}\label{sec:priv1}

We start by providing an upper bound on the norm $\norm{\theta_X^w - \theta_{X'}^w}_2$ for any two neighbouring datasets $X \simeq X'$. Using \eqref{eq:thetaXw} it is immediate  that:
\begin{equation}
\norm{\theta_X^w - \theta_{X'}^w}_2 \leq
\norm{(\Gamma^{1/2} \Phi)^\dagger} \norm{F_X - F_{X'}}_{2,\Gamma} \enspace.
\end{equation}
Thus, we need to bound $\norm{F_X - F_{X'}}_{2,\Gamma}$.

\begin{lemma}
Let $X \simeq X'$ be two neighbouring datasets of $m$ trajectories with $X = (x_1,\ldots, x_{m-1},x)$ and $X' = (x_1,\ldots,x_{m-1},x')$. Let $\Xz = (x_1,\ldots,x_{m-1})$. Let $\states_x$ (resp.\ $\states_{x'}$) denote the set of states visited by $x$ (resp.\ $x'$). Then we have
\begin{equation*}
\norm{F_X - F_{X'}}_{2,\Gamma} \leq \frac{\Rmax}{1 - \gamma}
\sqrt{\sum_{s \in \states_x \cup \states_{x'}}
\frac{w_s}{(|\Xz_s| + 1)^2}} \enspace.
\end{equation*}
\end{lemma}

Since the condition in Lemma~\ref{lem:Gaussian} needs to hold for any dataset $X'$ neighbouring $X$, we take the supremum of the bound above over all neighbours., which yields the following corollary.

\begin{corollary}\label{cor:LSw}
If $X$ is a dataset of trajectories, then the following holds for every neighbouring dataset $X' \simeq X$:
\begin{align*}
\norm{F_X - F_{X'}}_{2,\Gamma}
&\leq
\frac{\Rmax}{1-\gamma}
\sqrt{\sum_{s \in \states} \frac{w_s}{\max\{|X_s|,1\}^2}}
\enspace.
\end{align*}
\end{corollary}

Using this result we see that in order to satisfy item (a) of Lemma~\ref{lem:Gaussian} we can choose a noise variance satisfying:
\begin{equation}\label{eq:sigmaX}
\sigma_X \geq \frac{\alpha \Rmax \norm{(\Gamma^{1/2} \Phi)^{\dagger}}}{1-\gamma} \sqrt{\sum_{s \in \states} \frac{w_s}{\max\{|X_s|,1\}^2}} \enspace,
\end{equation}
where only the last multiplicative term depends on the dataset $X$, and the rest can be regarded as a constant that depends on parameters of the problem which are either public or chosen by the user, and will not change for a neighbouring dataset $X'$. Thus, we are left with a lower bound expressible as $\sigma_X \geq C \sqrt{\varphi_X^w}$, where $\varphi_X^w = \sum_s (w_s / \max\{|X_s|,1\}^2)$ only depends on the dataset $X$ through its \emph{signature} $\sig{X} \in \N^{\states}$ given by the number of times each state appears in the trajectories of $X$: $\sig{X}(s) = |X_{s}|$. Accordingly, we write $\varphi_X^w = \varphi^w(\sig{X})$, where $\varphi^w : \N^{\states} \to \R$ is the function
\begin{equation}\label{eq:varphi}
\varphi^w(v) = \sum_s \frac{w_s}{\max\{v_s,1\}^2} \enspace.
\end{equation}

The signatures of two neighbouring datasets $X \simeq X'$ satisfy $\norm{\sig{X} - \sig{X'}}_{\infty} \leq 1$ because replacing a single trajectory can only change by one the number of first visits to any particular state. Thus, assuming we have a function $\psi : \N^{\states} \to \R$ satisfying $\psi^w(v) \geq \varphi^w(v)$ and $|\ln(\psi^w(v)) - \ln(\psi^w(v'))| \leq \beta$ for all $v, v' \in \N^\states$ with $\norm{v - v'}_\infty \leq 1$, we can take $\sigma_X = C \sqrt{\psi^w(\sig{X})}$. This variance clearly satisfies the conditions of Lemma~\ref{lem:Gaussian} since
\begin{equation*}
|\ln(\sigma_X^2) - \ln(\sigma_{X'}^2)| = |\ln(\psi^w(\sig{X})) - \ln(\psi^w(\sig{X'}))| \leq \beta \enspace.
\end{equation*}
The function $\psi^w$  is known as a \emph{$\beta$-smooth upper bound} of $\varphi^w$, and the following result provides a tool for constructing such functions.

\begin{lemma}[\citet{nissim2007smooth}]\label{lem:betasmooth}
Let $\varphi : \N^{\states} \to \R$. For any $k \geq 0$ let $\varphi_k(v) = \max_{\norm{v - v'}_\infty \leq k} \varphi(v')$. Given $\beta > 0$, the smallest $\beta$-smooth upper bound of $\varphi$ is the function
\begin{equation}
\psi(v) = \sup_{k \geq 0} \left(e^{-k \beta} \varphi_k(v) \right) \enspace.
\end{equation}
\end{lemma}

For some functions $\varphi$, the upper bound $\psi$ can be hard to compute or even approximate~\citep{nissim2007smooth}. Fortunately, in our case a simple inspection of \eqref{eq:varphi} reveals that $\varphi_k^w(v)$ is easy to compute. In particular, the following lemma implies that $\psi^w(v)$ can be obtained in time $O(N \norm{v}_\infty)$.

\begin{lemma}
The following holds for every $v \in \N^{\states}$:
\begin{equation*}
\varphi_k^w(v) = \sum_{s \in \states} \frac{w_s}{\max\{v_s - k,1\}^2} \enspace.
\end{equation*}
Furthermore, for every $k \geq \norm{v}_\infty - 1$ we have $\varphi_k^w(v) = \sum_s w_s$.
\end{lemma}

Combining the last two lemmas, we see that the quantity $\psi_X^w$ computed in DP-LSW is in fact a $\beta$-smooth upper bound to $\varphi_X^w$. Because the variance $\sigma_X$ used in DP-LSW can be obtained by plugging this upper bound into \eqref{eq:sigmaX}, the two conditions of Lemma~\ref{lem:Gaussian} are satisfied. This completes the proof of the main result of this section:

\begin{theorem}
Algorithm DP-LSW is $(\varepsilon,\delta)$-differentially private.
\end{theorem}

Before proceeding to the next privacy analysis, note that Corollary~\ref{cor:LSw} is the reason why a mechanism with output perturbations proportional to the global sensitivity is not sufficient in this case. The bound there says that if in the worst case we can find datasets of an arbitrary size $m$ where some states are visited few (or zero) times, then the global sensitivity will not vanish as $m \to \infty$. Hence, the utility of such algorithm would not improve with the size of the dataset. The smoothed sensitivity approach works around this problem by adding large noise to these datasets, but adding much less noise to datasets where each state appears a sufficient number of times. Corollary~\ref{cor:LSw} also provides the basis for efficiently computing smooth upper bounds to the local sensitivity. In principle, condition (b) in Lemma~\ref{lem:Gaussian} refers to any dataset neighbouring $X$, of which there are uncountably many because we consider real rewards. Bounding the local sensitivity in terms of the signature reduces this to finitely many ``classes'' of neighbours, and the form of the bound in Corollary~\ref{cor:LSw} makes it possible to apply Lemma~\ref{lem:betasmooth} efficiently.

\subsection{Privacy Analysis of DP-LSL}\label{sec:priv2}

The proof that  DP-LSL is differentially private follows the same strategy as for DP-LSW. We start with a lemma that bounds the local sensitivity of $\theta_X^\lambda$ for pairs of neighbouring datasets $X \simeq X'$. We use the notation $\one_{s \in x}$ for an indicator variable that is equal to one when state $s$ is visited within trajectory $x$.

\begin{lemma}
Let $X \simeq X'$ be two neighbouring datasets of $m$ trajectories with $X = (x_1,\ldots, x_{m-1},x)$ and $X' = (x_1,\ldots,x_{m-1},x')$. Let $F_x \in \R^{\states}$ (resp.\ $F_{x'} \in \R^{\states}$) be the vector given by $F_{x}(s) = F_{x,s}$ (resp.\ $F_{x'}(s) = F_{x',s}$). Define diagonal matrices $\Gamma_\rho, \Delta_{x,x'} \in \R^{\states \times \states}$ given by $\Gamma_{\rho}(s,s) = \rho_s$ and $\Delta_{x,x'}(s,s) = \one_{s \in x} - \one_{s \in x'}$. If the regularization parameter satisfies $\lambda > \norm{\Phi^\top \Delta_{x,x'} \Gamma_{\rho} \Phi}$, then:
\begin{equation*}
\frac{\norm{\theta_X^\lambda - \theta_{X'}^\lambda}_2}{2} \leq
\frac{\bignorm{\left(\Delta_{x,x'} \Phi \theta_X^\lambda - F_x + F_{x'} \right)^\top \Gamma_\rho \Phi}_2}{\lambda - \norm{\Phi^\top \Delta_{x,x'} \Gamma_{\rho} \Phi}}  \enspace.
\end{equation*}
\end{lemma}

As before, we need to consider the supremum of the bound over all possible neighbours $X'$ of $X$. In particular, we would like to get a bound whose only dependence on the dataset $X$ is  through the signature $\sig{X}$. This is the purpose of the following corollary:

\begin{corollary}\label{cor:LSl}
Let $X$ be a dataset of trajectories and suppose $\lambda > \norm{\Phi}^2 \norm{\rho}_\infty$. Then the following holds for every neighbouring dataset $X' \simeq X$:
\begin{equation*}
\norm{\theta_X^\lambda - \theta_{X'}^\lambda}_2 \leq
\frac{2 \Rmax \norm{\Phi}}{(1-\gamma) (\lambda - \norm{\Phi}^2 \norm{\rho}_\infty)} \sqrt{\varphi_X^\lambda} \enspace,
\end{equation*}
where
\begin{equation*}
\varphi_X^\lambda = \left(\frac{\norm{\Phi} \norm{\rho}_\infty}{\sqrt{2 \lambda}} \sqrt{\sum_{s \in \states} \rho_s |X_s|} + \norm{\rho}_2\right)^2 \enspace.
\end{equation*}
\end{corollary}

By the same reasoning of Section~\ref{sec:priv1}, as long as the regularization parameter is larger than $\norm{\Phi}^2 \norm{\rho}_\infty$, a differentially private algorithm can be obtained by adding to $\theta_X^\lambda$ a Gaussian perturbation with a variance satisfying
\begin{equation*}
\sigma_X \geq \frac{2 \alpha \Rmax \norm{\Phi}}{(1-\gamma) (\lambda - \norm{\Phi}^2 \norm{\rho}_\infty)} \sqrt{\varphi_X^\lambda}
\end{equation*}
and the second condition of Lemma~\ref{lem:Gaussian}. This second requirement can be achieved by computing a $\beta$-smooth upper bound of the function $\varphi^\lambda : \N^\states \to \R$ given by
\begin{equation*}
\varphi^\lambda(v) = \left(\frac{\norm{\Phi} \norm{\rho}_\infty}{\sqrt{2 \lambda}} \sqrt{\sum_{s \in \states} \rho_s \max\{v_s, m\}} + \norm{\rho}_2\right)^2 \enspace.
\end{equation*}
When going from $\varphi_X^\lambda$ to $\varphi^\lambda(v)$ we  substituted $|X_s|$ by $\max\{v_s,m\}$ to reflect the fact that any state cannot be visited by more than $m$ trajectories in a dataset $X$ of size $m$. It turns out that in this case the function $\varphi^\lambda_k(v) = \max_{\norm{v - v'}_\infty \leq k} \varphi^\lambda(v')$ arising in Lemma~\ref{lem:betasmooth} is also easy to compute.

\begin{lemma}\label{lem:phikl}
For every $v \in \N^{\states}$, $\varphi_k^\lambda(v)$ is equal to:
\begin{equation*}
%\varphi_k^\lambda(v) =
\left(\frac{\norm{\Phi} \norm{\rho}_\infty}{\sqrt{2 \lambda}} \sqrt{\sum_{s \in \states} \rho_s \max\{v_s + k, m\}} + \norm{\rho}_2\right)^2 \enspace.
\end{equation*}
Furthermore, for every $k \geq m - \min_s v_s$ we have $\varphi_k^\lambda(v) = \left(\frac{\norm{\Phi} \norm{\rho}_\infty \sqrt{m}}{\sqrt{2 \lambda}} \sqrt{\sum_{s \in \states} \rho_s} + \norm{\rho}_2\right)^2$.
\end{lemma}

Finally,  in view of Lemma~\ref{lem:betasmooth}, Corollary~\ref{cor:LSl}, and Lemma~\ref{lem:phikl}, the variance of the noise perturbation in DP-LSL satisfies the conditions of Lemma~\ref{lem:Gaussian}, so we have proved the following.

\begin{theorem}
Algorithm DP-LSL is $(\varepsilon,\delta)$-differentially private.
\end{theorem}
\section{Utility Analysis}\label{sec:util}

Because the promise of differential privacy has to hold for any possible pair of neighbouring datasets $X \simeq X'$, the analysis in previous section does not assume any generative model for the input dataset $X$. However, in practical applications we expect $X = (x_1,\ldots,x_m)$ to contain multiple trajectories sampled from the same policy on the same MDP. The purpose of this section is to show that when the trajectories $x_i$ are i.i.d.\ the utility of our differentially private algorithms increases as $m \to \infty$. In other words, when the input dataset grows, the amount of noise added by our algorithms decreases, thus leading to more accurate estimates of the value function. This matches the intuition that when outputting a fixed number of parameters, using data from more users to estimate these parameters leads to a smaller individual contributions from each user, and makes the privacy constraint easier to satisfy.

To measure the utility of our DP algorithms we shall bound the difference in empirical risk between the private and non-private parameters learned from a given dataset. That is, we want to show that the quantity
$\E_{X,\eta}[J_X^\bullet(\hat{\theta}_X^\bullet) - J_X^\bullet(\theta_X^\bullet)]$ vanishes as $|X| = m \to \infty$, for both $\bullet = w$ and $\bullet = \lambda$. The first theorem bounds the expected empirical excess risk of DP-LSW. The bound contains two terms: one vanishes as $m \to \infty$, and the other reflects the fact that states which are never visited pose a problem to stability. The proof is deferred to Appendix~\ref{ap:sec:util1}.

\begin{theorem}
Let $\states_0 = \{ s \in \states | p_s = 0 \}$ and $S_+ = \states \setminus \states_0$. Let
$C = \alpha \Rmax \norm{(\Gamma^{1/2} \Phi)^{\dagger}}
\norm{\Gamma^{1/2} \Phi}_F
/ (1-\gamma)$. Suppose $\beta \leq 1/2$. Then $\E_{X,\eta}[J_X^w(\hat{\theta}_X^w) - J_X^w(\theta_X^w)] $ is upper bounded by:
\begin{equation*}
C^2 \left( \sum_{s \in \states_0} w_s + 6 \sum_{s \in \states_+} w_s \left(\frac{1}{p_s^2 m^2} + \beta^2 \left(1 - \frac{\beta p_s}{2}\right)^m \right) \right) \enspace.
\end{equation*}
\end{theorem}

Note the above bound depends on the dimension $d$ through $\beta$ and $\norm{\Gamma^{1/2} \Phi}_F$. In terms of the size of the dataset, we can get excess risk bounds that decreases quadratically with $m$ by assuming that either all states are visited with non-zero probability or the user sets the regression weights so that such states do not contribute to $\theta_X^w$.

\begin{corollary}
If $w_s = 0$ for all $s \in \states_0$, then $\E_{X,\eta}[J_X^w(\hat{\theta}_X^w) - J_X^w(\theta_X^w)] = O(1/m^2)$.
\end{corollary}

A similar theorem can be proved for DP-LSL. However, in this case the statement of the bound is complicated by the appearance of co-occurrence probabilities of the form $\P_x[s \in x \wedge s' \in x]$ and $\P_x[s \in x \wedge s' \notin x]$. Here we only state the main corollary of our result; the full statement and the corresponding proofs are presented in Appendix~\ref{ap:sec:util2}. This corollary is obtained by assuming the regularization parameter is allowed to grow with $m$, and stresses the tensions in selecting an adequate regularization schedule.

\begin{corollary}
Suppose $\lambda = \omega(1)$ with respect to $m$. Then we have
$\E_{X,\eta}[J_X^\lambda(\hat{\theta}_X^\lambda) - J_X^\lambda(\theta_X^\lambda) ] = O(1/\lambda m + 1/\lambda^2 + m/\lambda^3)$.
%\left(\frac{1}{\lambda m} + \frac{1}{\lambda^2} + \frac{m}{\lambda^3} \right)
\end{corollary}

Note that taking $\lambda = \Theta(m)$ we get a bound on the excess risk of order $O(1/m^2)$. However, if we want the regularization term in $J_X^\lambda(\theta)$ to vanish as $m \to \infty$ we need $\lambda = o(m)$. We shall see importance of this trade-off in our experiments.

\section{Experiments}\label{sec:exp}

\begin{figure}[t]
\begin{center}
\begin{subfigure}[b]{0.24\textwidth}
\includegraphics[width=\textwidth]{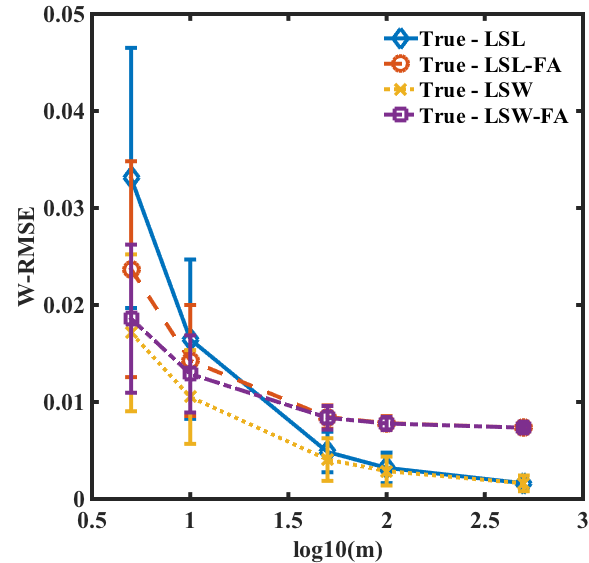}
%\caption{BLA}
\label{fig:exp1}
\end{subfigure}
\begin{subfigure}[b]{0.24\textwidth}
\includegraphics[width=\textwidth]{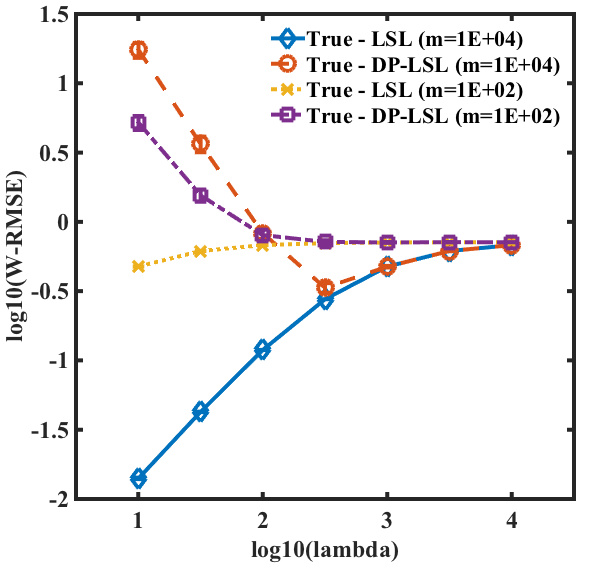}
%\caption{BLA}
\label{fig:exp2}
\end{subfigure}
\begin{subfigure}[b]{0.24\textwidth}
\includegraphics[width=\textwidth]{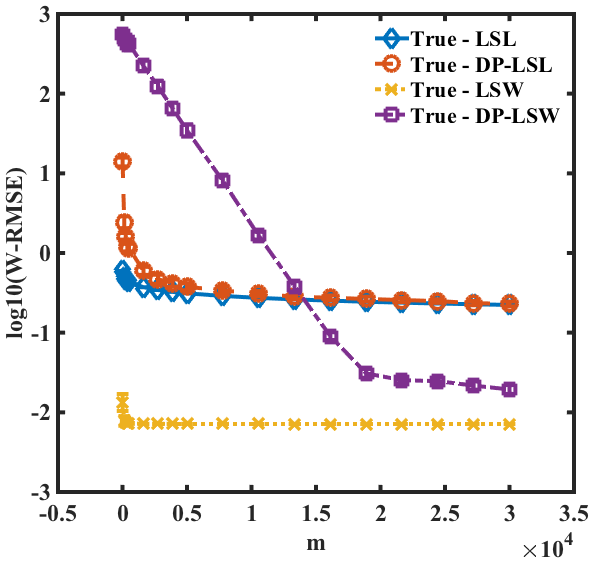}
%\caption{BLA}
\label{fig:exp3}
\end{subfigure}
\begin{subfigure}[b]{0.24\textwidth}
\includegraphics[width=\textwidth]{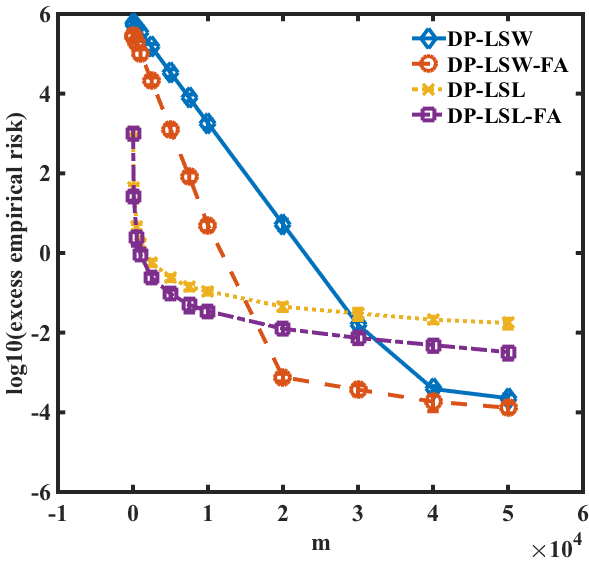}
%\caption{BLA}
\label{fig:exp4}
\end{subfigure}
\vspace*{-2em}
\caption{Empirical comparison of differentially private and non-private algorithms}\label{fig:exp}
\end{center}
\vspace*{-1em}
\end{figure}

In this section we illustrate the behaviour of the proposed algorithms on synthetic examples. The domain we use consists of a chain of $N$ states, where in each state the agent has some probability $p$ of staying and probability $(1-p)$ of advancing to its right. There is a reward of $1$ when the agent reaches the final, absorbing state, and $0$ for all other states.  While this is a toy example, it illustrates the typical case of policy evaluation in the medical domain, where patients tend to progress through stages of recovery at different speeds, and past states are not typically revisited (partly because in the medical domain, states contain historic information about past treatments). Trajectories are drawn by starting in an initial state distribution
%(in our case, uniform)
and generating state-action-reward transitions according to the described probabilities until the absorbing state is reached. Trajectories are harvested in a batch, and the same batches are processed by all algorithms.

We experiment with both a tabular representation of the value function, as well as with function approximation.  In the latter case, we simply aggregate pairs of adjacent states, which are hence forced to take the same value. We compared the proposed private algorithms DP-LSW and DP-LSL with their non-private equivalents LSW and LSL. The performance measure used is average root mean squared error over the state space. %evaluated with respect to the initial state distribution.
The error is obtained by comparing the state values estimated by the learning algorithms against the exact values obtained by exact, tabular dynamic programming.
%(which is feasible in this case, given that the environment is known).
Standard errors computed over 20 independent runs are included.

The main results are summarized in Fig.~\ref{fig:exp}, for an environment with $N=40$ states, $p=0.5$, discount $\gamma = 0.99$, and for the DP algorithms, $\varepsilon=0.1$ and $\delta=0.1$. In general, these constants should be chosen depending on the privacy constraints of the domain. Our theoretical results explain the expected effect of these choices on the privacy-utility trade-off so we do not provide extensive experiments with different values. 

The left plot in Fig.~\ref{fig:exp} compares the non-private LSL and LSW versions of Monte Carlo evaluation, in the tabular and function approximation case. As can be seen, both algorithms are very stable and converge to the same solution, but LSW converges faster. The second plot compares the performance of all algorithms in the tabular case, over a range of regularization parameters, for two different batch sizes.  The third plot compares the expected RMSE of the algorithms when run with state aggregation, as a function of batch size.  As can be seen, the DP algorithms converge to the same solutions as the non-private corresponding versions for large enough batch sizes. Interestingly, the two proposed approaches serve different needs.  The LSL algorithms work better with small batches of data, whereas the LSW approach is preferable with large  batches.  From an empirical point of view, the trade-off between accuracy and privacy in the DP-LSL algorithm should be done by setting a regularization schedule proportional to $\sqrt{m}$.  While the theory suggests it is not the best schedule in terms of excess empirical risk, it achieves the best overall accuracy.

Finally, the last figure shows excess empirical risk as a function of the batch size.  Interestingly, more aggressive function approximation helps both differentially private algorithms converge faster. This is intuitive, since using the same data to estimate fewer parameters means the effect of each individual trajectory is already obscured by the function approximation. Decreasing the number of parameters of the function approximator, $d$, increases $\beta$, which lowers the smooth sensitivity bounds. In medical applications, one expects to have many attributes measured about patients, and to need aggressive function approximation in order to provide generalization. This result tells us that differentially private algorithms should be favoured in this case as well.

Overall, the empirical results are very promising, showing that especially as batch size increases, the noise introduced by the DP mechanism decreases rapidly, and these algorithms provide the same performance but with the additional privacy guarantees.

\section{Conclusion}\label{sec:conclu}

We present the first differentially private algorithms for policy evaluation in the full MDP setting. Our algorithms are built on top of established Monte Carlo methods, and come with utility guarantees showing that the cost of privacy diminishes as training batches get larger. The smoothed sensitivity framework is a key component of our analyses, which differ from previous works on DP mechanisms for ERM and bandits problems in two substantial ways. The first, we consider optimizations with non-Lipschitz loss functions, which prevents us from using most of the established techniques for analyzing privacy and utility in ERM algorithms and complicates some parts of our analysis. In particular, we cannot leverage the tight utility analysis of~\citep{jain2014near} to get dimension independent bounds. Second, and more importantly, the natural model of neighbouring datasets for policy evaluation involves replacing a whole trajectory. This implies that neighbouring datasets can differ in multiple regression targets, which is quite different from the usual supervised learning approach where neighbouring datasets can only change a single regression target. Our approach is also different from the on-line learning and bandits setting, where there is a single stream of experience and neighbouring datasets differ in one element of the stream. Note that this setting cannot be used naturally in the full MDP setup, because successive observations in a single stream are inherently correlated.

In future work we plan to extend our techniques in two directions. First, we would like to design DP policy evaluation methods based on temporal-difference learning~\citep{Sutton1988}. Secondly, we will tackle the control case, where policy evaluation is often used as a sub-routine, e.g.~as in actor-critic methods. We also plan to evaluate the current algorithms on patient data from an ongoing clinical study (in which case, errors cannot be estimated precisely, because the right answer is not known).

\bibliography{paper}

\begin{thebibliography}{18}
\providecommand{\natexlab}[1]{#1}
\providecommand{\url}[1]{\texttt{#1}}
\expandafter\ifx\csname urlstyle\endcsname\relax
  \providecommand{\doi}[1]{doi: #1}\else
  \providecommand{\doi}{doi: \begingroup \urlstyle{rm}\Url}\fi

\bibitem[Chaudhuri and Monteleoni(2009)]{chaudhuri2009privacy}
Kamalika Chaudhuri and Claire Monteleoni.
\newblock Privacy-preserving logistic regression.
\newblock In \emph{Advances in Neural Information Processing Systems}, pages
  289--296, 2009.

\bibitem[Chaudhuri et~al.(2011)Chaudhuri, Monteleoni, and
  Sarwate]{chaudhuri2011differentially}
Kamalika Chaudhuri, Claire Monteleoni, and Anand~D Sarwate.
\newblock Differentially private empirical risk minimization.
\newblock volume~12. JMLR. org, 2011.

\bibitem[Dwork(2006)]{dwork2006differential}
Cynthia Dwork.
\newblock Differential privacy.
\newblock In \emph{Proceedings of the 33rd international conference on
  Automata, Languages and Programming-Volume Part II}, pages 1--12, 2006.

\bibitem[Dwork and Roth(2014)]{dwork2014algorithmic}
Cynthia Dwork and Aaron Roth.
\newblock The algorithmic foundations of differential privacy.
\newblock \emph{Foundations and Trends in Theoretical Computer Science},
  9\penalty0 (3-4):\penalty0 211--407, 2014.

\bibitem[Jain and Thakurta(2013)]{jain2013differentially}
Prateek Jain and Abhradeep Thakurta.
\newblock Differentially private learning with kernels.
\newblock In \emph{Proceedings of the 30th International Conference on Machine
  Learning (ICML-13)}, pages 118--126, 2013.

\bibitem[Jain and Thakurta(2014)]{jain2014near}
Prateek Jain and Abhradeep~Guha Thakurta.
\newblock (near) dimension independent risk bounds for differentially private
  learning.
\newblock In \emph{Proceedings of The 31st International Conference on Machine
  Learning}, pages 476--484, 2014.

\bibitem[Laurent and Massart(2000)]{laurent2000adaptive}
Beatrice Laurent and Pascal Massart.
\newblock Adaptive estimation of a quadratic functional by model selection.
\newblock \emph{Annals of Statistics}, pages 1302--1338, 2000.

\bibitem[McSherry and Talwar(2007)]{mcsherry2007mechanism}
Frank McSherry and Kunal Talwar.
\newblock Mechanism design via differential privacy.
\newblock In \emph{Foundations of Computer Science, 2007. FOCS'07. 48th Annual
  IEEE Symposium on}, pages 94--103. IEEE, 2007.

\bibitem[Mishra and Thakurta(2015)]{mishra2015}
Nikita Mishra and Abhradeep Thakurta.
\newblock Nearly optimal differentially private stochastic multi-arm bandits.
\newblock In \emph{UAI}, 2015.

\bibitem[Nissim et~al.(2007)Nissim, Raskhodnikova, and Smith]{nissim2007smooth}
Kobbi Nissim, Sofya Raskhodnikova, and Adam Smith.
\newblock Smooth sensitivity and sampling in private data analysis.
\newblock In \emph{Proceedings of the thirty-ninth annual ACM symposium on
  Theory of computing}, pages 75--84. ACM, 2007.

\bibitem[Nissim et~al.(2011)Nissim, Raskhodnikova, and
  Smith]{Nissim2011JounalVersion}
Kobbi Nissim, Sofya Raskhodnikova, and Adam Smith.
\newblock Smooth sensitivity and sampling in private data analysis, 2011.
\newblock URL \url{http://www.cse.psu.edu/~ads22/pubs/NRS07/}.

\bibitem[Rubinstein et~al.(2012)Rubinstein, Bartlett, Huang, and
  Taft]{rubinstein2012learning}
Benjamin~IP Rubinstein, Peter~L Bartlett, Ling Huang, and Nina Taft.
\newblock Learning in a large function space: Privacy-preserving mechanisms for
  svm learning.
\newblock \emph{Journal of Privacy and Confidentiality}, 4\penalty0
  (1):\penalty0 4, 2012.

\bibitem[Smith and Thakurta(2013)]{smith2013}
Adam Smith and Abhradeep Thakurta.
\newblock Nearly optimal algorithms for private online learning in
  full-information and bandit settings.
\newblock In \emph{NIPS}, 2013.

\bibitem[Sutton(1988)]{Sutton1988}
Richard~S. Sutton.
\newblock Learning to predict by the methods of temporal differences.
\newblock \emph{Machine Learning}, 3\penalty0 (1):\penalty0 9--44, 1988.

\bibitem[Sutton and Barto(1998)]{sutton1998reinforcement}
Richard~S Sutton and Andrew~G Barto.
\newblock \emph{Reinforcement learning: An introduction}.
\newblock MIT press, 1998.

\bibitem[Szepesv{\'a}ri(2010)]{szepesvari2010algorithms}
Csaba Szepesv{\'a}ri.
\newblock \emph{Algorithms for reinforcement learning}.
\newblock Morgan \& Claypool Publishers, 2010.

\bibitem[Thakurta and Smith(2013)]{thakurta2013differentially}
Abhradeep~Guha Thakurta and Adam Smith.
\newblock Differentially private feature selection via stability arguments, and
  the robustness of the lasso.
\newblock In \emph{Conference on Learning Theory}, pages 819--850, 2013.

\bibitem[Tossou and Dimitrakakis(2016)]{tossou:aaai2016}
Aristide C.~Y. Tossou and Christos Dimitrakakis.
\newblock Algorithms for differentially private multi-armed bandits.
\newblock In \emph{International Conference on Artificial Intelligence ({AAAI
  2016})}, 2016.

\end{thebibliography}
\bibliographystyle{plainnat}

\appendix
\section{Smoothed Gaussian Perturbation}\label{ap:sec:smthgauss}

A proof of Lemma 1 in the paper can be found in the pre-print \cite{Nissim2011JounalVersion}. For the sake of completeness, we provide here an elementary proof (albeit with slightly worse constants). In particular, we are going to prove the following.

\begin{lemma}\label{ap:lem:Gaussian}
Let $A$ be an algorithm that on input $X$ computes a vector $\mu_X \in \R^d$ deterministically and then outputs $Z_X \sim \mathcal{N}(\mu_X, \sigma_X^2 I)$, where $\sigma_X^2$ is a variance that depends on $X$. Let $\alpha = \alpha(\varepsilon, \delta) = 15 \sqrt{2 \ln(4/\delta)} / \varepsilon$ and $\beta = \beta(\varepsilon, \delta, d) = (2 \ln 2) \varepsilon / 5 (\sqrt{d} + \sqrt{2 \ln(4/\delta)})^2$.
%Suppose the following are satisfied for some $\varepsilon \leq 5$, $\delta \geq 2 / e$, and every pair of neighbouring datasets $X \sim X'$:
Suppose that $\varepsilon \leq 5$, $\delta$ and $d$ are such $\beta \leq \ln 2$, and the following are satisfied for every pair of neighbouring datasets $X \simeq X'$:
\begin{enumerate}
\item $\sigma_X \geq \alpha \norm{\mu_X - \mu_{X'}}_2$,
\item $|\ln(\sigma_X^2) - \ln(\sigma_{X'}^2)| \leq \beta$.
\end{enumerate}
Then $A$ is $(\varepsilon,\delta)$-differentially private.
\end{lemma}

We start with a simple characterization of $(\varepsilon,\delta)$-differential privacy that will be useful for our proof.

\begin{lemma}\label{ap:lem:Theta}
Let $A(X) = \theta_X \in \R^d$ be the output of a randomized algorithm on input $X$. Write $f_{\theta_X}(\theta)$ for the probability density of the output of $A$ on input $X$. Suppose that for every pair of neighbouring datasets $X \simeq X'$ there exists a measurable set $\Theta_{X,X'} \subset \R^d$ such that the following are satisfied:
\begin{enumerate}
\item $\P[\theta_X \notin \Theta_{X,X'}] \leq \delta$;
\item for all $\theta \in \Theta_{X,X'}$ we have $f_{\theta_X}(\theta) \leq e^{\varepsilon} f_{\theta_{X'}}(\theta)$.
\end{enumerate}
Then $A$ is $(\varepsilon,\delta)$-differentially private.
\end{lemma}
\begin{proof}
Fix a pair of neighbouring datasets $X \simeq X'$ and let $E \subseteq \R^d$ be any measurable set. Let $\Theta_{X,X'}$ be as in the statement and write $\Theta_{X,X'}^{\mathsf{c}} = \R^d \setminus \Theta_{X,X'}$. Using the assumptions on $\Theta_{X,X'}$ we see that
\begin{align*}
\P[\theta_X \in E]
&=
\P[\theta_X \in E \cap \Theta_{X,X'}] + \P[\theta_X \in E \cap \Theta_{X,X'}^\mathsf{c}] \\
&\leq
e^{\varepsilon} \P[\theta_{X'} \in E \cap \Theta_{X,X'}] + \delta \\
&\leq
e^{\varepsilon} \P[\theta_{X'} \in E] + \delta \enspace. \qedhere
\end{align*}
\end{proof}

Now we proceed with the proof of Lemma~\ref{ap:lem:Gaussian}. Let $X \simeq X'$ be two neighbouring datasets and let us write $Z_1 = Z_X$ and $Z_2 = Z_{X'}$ for simplicity. Thus, for $i = 1, 2$ we have that $Z_i \sim \cN(\mu_i,\sigma_i^2 I)$ are $d$-dimensional independent Gaussian random variables whose means and variances satisfy the assumptions of Lemma~\ref{ap:lem:Gaussian} for some $\varepsilon, \delta > 0$. The density function of $Z_i$ is denoted by $f_{Z_i}(z)$. In order to be able to apply Lemma~\ref{ap:lem:Theta} we want to show that the privacy loss between $Z_1$ and $Z_2$ defined as
\begin{equation}
L(z) = \ln \frac{f_{Z_1}(z)}{f_{Z_2}(z)}
\end{equation}
is bounded by $\varepsilon$ for all $z \in \Omega$, where $\Omega \subset \R^d$ is an event with probability at least $1 - \delta$ under $Z_1$.

We can start by identifying a candidate $\Omega$. Since $\Omega$ has to have high probability w.r.t.\ $Z_1$, it should contain $\mu_1$ because a ball around the mean is the event with the highest probability under a spherical Gaussian distribution (among those with the same Lebesgue measure). For technical reasons, instead of a ball we will take a slightly more complicated region, which for now we will parametrize by two quantities $a, b > 0$. The definition of this region will depend on the difference of means $\Delta = \mu_2 - \mu_1$:
\begin{equation}
\Omega = \Omega_{a} \cap \Omega_{b} = \{ z + \mu_1 \in \R^d \; | \; |\left< z, \Delta\right>| \leq a \} \cap \{ z + \mu_1 \in \R^d \; | \; \norm{z} \leq b \} \enspace.
\end{equation}

We need to choose $a$ and $b$ such that the probability $\P[Z_1 \notin \Omega] \leq \delta$, and for that we shall combine two different tail bounds. On the one hand, note that $Z = \left< Z_1 - \mu_1, \Delta \right> / (\sigma_1 \norm{\Delta}) \sim \cN(0,1)$ is a one dimensional standard Gaussian random variable and recall that for any $t \geq 0$:%\footnote{One could actually use slightly tighter tail bounds here at the price of making the computation more involved, but I don't think this will improve anything beyond constants.}
\begin{equation}
\P[ |Z| > t ] \leq 2 e^{-t^2 / 2} \enspace.
\end{equation}
On the other hand, $X = \norm{Z_1 - \mu_1}^2 / \sigma_1^2 \sim \chi^2_d$ follows a chi-squared distribution with $d$ degrees of freedom, for which is known \cite{laurent2000adaptive} that for all $t \geq 0$:
\begin{equation}
\P[ X > d + 2 \sqrt{d t} + 2 t ] \leq e^{-t} \enspace.
\end{equation}
To make our choices for $a$ and $b$ we can take them such that $\P[Z_1 \notin \Omega_a], \P[Z_1 \notin \Omega_b] \leq \delta/2$, since then by a union bound we will get
\begin{equation}
\P[Z_1 \notin \Omega] \leq \P[Z_1 \notin \Omega_A] + \P[Z_1 \notin \Omega_B] \leq \delta \enspace.
\end{equation}
Since $Z$ satisfies $|Z| \leq \sqrt{2 \ln (4 / \delta)}$ with probability at least $1 - \delta/2$, we can take
\begin{equation}
a = \sigma_1 \norm{\Delta} \sqrt{2 \ln \frac{4}{\delta}} = \sigma_1 \norm{\Delta} C_{\delta} \enspace.
\end{equation}
For $X$ we have that
%$X \leq d + 2 \sqrt{d \ln(2/\delta)} + 2 \ln (2/\delta) \leq d + 4 \sqrt{d} \ln(2/\delta)$, where the second inequality assumes that $\delta \leq 2 / e \approx 0.73$ (note this is not a big restriction since applications require small $\delta$ in order to yield meaningful results).
$d + 2 \sqrt{d \ln(2/\delta)} + 2 \ln (2/\delta) \leq d + 2 \sqrt{2 d \ln(2/\delta)} + 2 \ln(2/\delta) = (\sqrt{d} + \sqrt{2 \ln(2/\delta)})^2$.
Hence, we choose
\begin{equation}
%B = \sigma_1 \sqrt{d + 4 \sqrt{d} \ln \frac{2}{\delta}} = \sigma_1 D_{\delta} \enspace.
b = \sigma_1 (\sqrt{d} + \sqrt{2 \ln(2/\delta)}) = \sigma_1 D_{\delta} \enspace.
\end{equation}

%\borja{The above can be improved by using the upper bound $d + 2 \sqrt{d \ln(2/\delta)} + 2 \ln (2/\delta) \leq d + 2 \sqrt{2 d \ln(2/\delta)} + 2 \ln(2/\delta) = (\sqrt{d} + \sqrt{2 \ln(2/\delta)})^2$, which requires no assumption on $\delta$ and leads to a choice of $D_\delta = \sqrt{d} + \sqrt{2 \ln(2/\delta)}$.}

Fixing this choice of $\Omega$, we now proceed to see under what conditions on $\sigma_1$ and $\sigma_2$ we can get $L(z) \leq \varepsilon$ for all $z \in \Omega$. We start by expanding the definition of $L(z)$ to get
\begin{equation}
L(z) = \frac{d}{2} \ln \frac{\sigma_2^2}{\sigma_1^2} + \frac{\norm{\mu_2 - z}^2}{2 \sigma_2^2} - \frac{\norm{\mu_1 - z}^2}{2 \sigma_1^2} \enspace.
\end{equation}
The easiest thing to do is to separate this quantity into several parts and insist on each part being at most a fraction of $\varepsilon$. To simplify calculations we will just require that each part is at most $\epsilon = \varepsilon/5$. This reasoning applied to the first term shows that we must satisfy
\begin{equation}\label{ap:eq:s2divs1}
\frac{\sigma_2^2}{\sigma_1^2} \leq e^{2 \epsilon / d} \enspace.
\end{equation}
Note that this becomes more restrictive as $\epsilon \approx 0$ or $d \to \infty$, in which case we have $e^{\epsilon / d} \approx 1$.

Next we look at the second part and write $z = z' + \mu_1$ because this is the form of the vectors in $\Omega$. With some algebra we get:
\begin{equation}
\frac{\norm{\mu_2 - (z' + \mu_1)}^2}{2 \sigma_2^2} - \frac{\norm{\mu_1 - (z' + \mu_1)}^2}{2 \sigma_1^2} =
\frac{\norm{\Delta}^2 + \norm{z'}^2 - 2 \left<z', \Delta\right>}{2 \sigma_2^2} -
\frac{\norm{z'}^2}{2 \sigma_1^2} \enspace.
\end{equation}
To further decompose this quantity we write $z' \in \R^d$ as $z' = z_p + z_o$, where $z_p = \Delta \left<z', \Delta\right> / \norm{\Delta}^2$ is the orthogonal projection of $z$ onto the line spanned by the vector $\Delta$, and $z_o$ is the corresponding orthogonal complement. Pythagora's Theorem implies $\norm{z'}^2 = \norm{z_p}^2 + \norm{z_o}^2$, and the RHS in the above expression is equal to
\begin{equation}
\frac{\norm{\Delta}^2}{2 \sigma_2^2}
- \frac{\left<z', \Delta\right>}{\sigma_2^2}
+ \frac{|\left<z', \Delta\right>|^2}{2 \norm{\Delta}^2} \left(\frac{1}{\sigma_2^2} - \frac{1}{\sigma_1^2}\right)
+ \frac{\norm{z_o}^2}{2} \left(\frac{1}{\sigma_2^2} - \frac{1}{\sigma_1^2}\right)
\enspace.
\end{equation}
Now note that the last two terms can be upper bounded by zero if $\sigma_1 \leq \sigma_2$, but need to be taken into account otherwise. Furthermore, if it were the case that $\sigma_1 \gg \sigma_2 \approx 0$, then these terms could grow unboundedly. Thus we shall require that a bound of the form
\begin{equation}\label{ap:eq:s1divs2}
\frac{\sigma_1^2}{\sigma_2^2} \leq \gamma \enspace,
\end{equation}
holds for some $\gamma \geq 1$ to be specified later. Nonetheless, we observe that under this assumption
\begin{equation}
\frac{1}{\sigma_2^2} - \frac{1}{\sigma_1^2} \leq 
\frac{\gamma - 1}{\sigma_1^2} \enspace.
\end{equation}
Furthermore, $z \in \Omega$ implies $\norm{z_o}^2 \leq \norm{z'}^2 = \norm{z - \mu_1} \leq b^2$ and $|\left<z',\Delta\right>|^2 = |\left<z - \mu_1,\Delta\right>|^2 \leq a^2$. Thus we see that
\begin{equation}
\frac{|\left<z', \Delta\right>|^2}{2 \norm{\Delta}^2} \left(\frac{1}{\sigma_2^2} - \frac{1}{\sigma_1^2}\right)
\leq
\frac{C_{\delta}^2 (\gamma - 1)}{2} \enspace,
\end{equation}
and
\begin{equation}
\frac{\norm{z_o}^2}{2} \left(\frac{1}{\sigma_2^2} - \frac{1}{\sigma_1^2}\right)
\leq
\frac{D_{\delta}^2 (\gamma - 1)}{2} \enspace.
\end{equation}
By requiring that each of these bounds is at most $\epsilon$ we obtain the following constraint for $\gamma$:
\begin{equation}
\gamma \leq 1 + \frac{2 \epsilon}{\max\{C_\delta^2, D_\delta^2\}} \enspace,
\end{equation}
which can be satisfied by taking, for example:
\begin{equation}
%\gamma = 1 + \frac{2 \epsilon}{d + 4 \sqrt{d} \ln(4/\delta)} \enspace.
\gamma = 1 + \frac{2 \epsilon}{\left(\sqrt{d} + \sqrt{2 \ln(4/\delta)}\right)^2} \enspace.
\end{equation}
Note that for fixed $\delta$, small $\epsilon$ and/or large $d$ this choice of $\gamma$ will make \eqref{ap:eq:s1divs2} behave much like the bound \eqref{ap:eq:s2divs1} we assumed above for $\sigma_2^2 / \sigma_1^2$. In fact, using that $1 + x \geq e^{x \ln 2}$ for all $0 \leq x \leq 1$ we see that \eqref{ap:eq:s1divs2} can be satisfied if $2 \epsilon/(\sqrt{d} + \sqrt{2 \ln(4/\delta)})^2 \leq 1$ and
\begin{equation}\label{ap:eq:s1divs2exp}
\frac{\sigma_1^2}{\sigma_2^2} \leq \exp\left(\frac{(2 \ln 2) \epsilon}{\left(\sqrt{d} + \sqrt{2 \ln (4/\delta)}\right)^2}\right) \enspace.
\end{equation}
From here it is immediate to see that if the second condition $|\ln(\sigma_1^2) - \ln(\sigma_2^2)| \leq \beta$ in Lemma~\ref{ap:lem:Gaussian} is satisfied, then \eqref{ap:eq:s2divs1} and \eqref{ap:eq:s1divs2exp} are both satisfied.

%\borja{With the new $D_\delta$ above this can be changed to $\gamma = 1 + 2 \epsilon / (\sqrt{d} + \sqrt{2 \ln(4/\delta)})^2$.}

The missing ingredient to show that $L(z) \leq \varepsilon$ for all $z \in \Omega$ is an absolute lower bound on $\sigma_1$. This will follow from bounding the remaining terms in $L(z)$ as follows:
\begin{align}
\frac{\norm{\Delta}^2}{2 \sigma_2^2} - \frac{\left<z', \Delta\right>}{\sigma_2^2}
&\leq
\frac{\norm{\Delta}^2 + 2 \sigma_1 \norm{\Delta} C_{\delta}}{2 \sigma_2^2} \\
&\leq
\frac{\gamma}{2} \frac{\norm{\Delta}^2 + 2 \sigma_1 \norm{\Delta} C_{\delta}}{\sigma_1^2} \\
&\leq
\frac{3}{2} \frac{\norm{\Delta}^2 + 2 \sigma_1 \norm{\Delta} C_{\delta}}{\sigma_1^2} \\
&=
\frac{3 \norm{\Delta}^2}{2 \sigma_1^2} + \frac{3 \norm{\Delta} C_{\delta}}{\sigma_1}
\enspace,
\end{align}
where we used that $\epsilon \leq 1$ implies $\gamma \leq 3$. If we require each of these two terms to be at most $\epsilon$, we obtain the constraint:
\begin{equation}
\sigma_1 \geq \norm{\Delta} \max \left\{ \sqrt{\frac{3}{2 \epsilon}}, \frac{3 C_\delta}{\epsilon} \right\} =
\frac{3 \norm{\Delta} C_\delta}{\epsilon} \enspace.
\end{equation}
To conclude the proof just note that the above bound can be rewritten as $\sigma_1 \geq \alpha \norm{\Delta}$, which is precisely the first condition in Lemma~\ref{ap:lem:Gaussian}.

\section{Privacy Analysis of DP-LSW}\label{ap:sec:priv1}

\begin{lemma}\label{ap:lem:lsxxp}
Let $X \simeq X'$ be two neighbouring datasets of $m$ trajectories with $X = (x_1,\ldots, x_{m-1},x)$ and $X' = (x_1,\ldots,x_{m-1},x')$. Let $\Xz = (x_1,\ldots,x_{m-1})$. Let $\states_x$ (resp.\ $\states_{x'}$) denote the set of states visited by $x$ (resp.\ $x'$). Then we have
\begin{equation*}
\norm{F_X - F_{X'}}_{2,\Gamma} \leq \frac{\Rmax}{1 - \gamma} \sqrt{\sum_{s \in \states_x \cup \states_{x'}} \frac{w_s}{(|\Xz_s| + 1)^2}} \enspace.
\end{equation*}
\end{lemma}
\begin{proof}
We start by noting that if $s \in \states \setminus (\states_x \cup \states_{x'})$, then $F_{X,s} = F_{X',s}$. In the case $s \in \states_{x} \cup \states_{x'}$ we can write $F_{X,s} = (|\Xz_s| F_{\Xz,s} + F_{x,s})/(|\Xz_s| + 1)$. Using a symmetric expression for $F_{X',s}$ we see that in this case
\begin{equation*}
|F_{X,s} - F_{X',s}|
=
\frac{1}{|\Xz_s|+1} |F_{x,s} - F_{x',s}|
\leq 
\frac{1}{|\Xz_s|+1} \max\{F_{x,s}, F_{x',s}\}
\leq
\frac{1}{|\Xz_s|+1} \frac{\Rmax}{1-\gamma} \enspace,
\end{equation*}
where we used that $0 \leq F_{x,s} \leq \Rmax/(1-\gamma)$ for all $s$ and $x$. When $s \in \states_x \setminus \states_{x'}$ we can use the same expression as before for $F_{X,s}$ and write $F_{X',s} = F_{\Xz,s}$. A similar argument as in the previous case then yields
\begin{equation*}
|F_{X,s} - F_{X',s}|
=
\frac{1}{|\Xz_s|+1} |F_{x,s} - F_{\Xz,s}|
\leq
\frac{1}{|\Xz_s|+1} \frac{\Rmax}{1-\gamma} \enspace.
\end{equation*}
Note the same bound also holds for the case $s \in \states_{x'} \setminus \states_{x}$. Finally, since we have seen that the same bound holds for all $s \in \states_x \cup \states_{x'}$, we obtain
\begin{equation*}
\sum_{s \in \states} w_s (F_{X,s} - F_{X',s})^2 \leq \frac{\Rmax^2}{(1-\gamma)^2} \sum_{s \in \states_x \cup \states_{x'}} \frac{w_s}{(|\Xz_s| + 1)^2} \enspace,
\end{equation*}
which yields the desired bound.
\end{proof}

\begin{corollary}
If $X$ is a dataset of trajectories, then the following holds for every neighbouring dataset $X' \simeq X$:
\begin{equation*}
\norm{F_X - F_{X'}}_{2,\Gamma} \leq
\frac{\Rmax}{1-\gamma}
\sqrt{\sum_{s \in \states} \frac{w_s}{\max\{|X_s|,1\}^2}}
\enspace.
\end{equation*}
\end{corollary}
\begin{proof}
Using the notation from Lemma~\ref{ap:lem:lsxxp} we observe that $|X_s| = |\Xz_s| + 1$ if $s \in \states_x$, and $|X_s| = |\Xz_s|$ if $s \notin \states_x$. Therefore, the following holds for any trajectories $x, x'$:
\begin{equation*}
\sum_{s \in \states_x \cup \states_{x'}} \frac{w_s}{(|\Xz_s|+1)^2} \leq
\sum_{s \in \states} \frac{w_s}{(|\Xz_s|+1)^2} =
\sum_{s \in \states_x} \frac{w_s}{|X_s|^2} +
\sum_{s \in \states \setminus \states_x} \frac{w_s}{(|X_s|+1)^2} \leq
\sum_{s \in \states_X} \frac{w_s}{|X_s|^2} +
\sum_{s \in \states \setminus \states_X} w_s
\enspace,
\end{equation*}
where $\states_X$ denotes the set of states visited by at least one trajectory from $X$. Since $s \notin \states_X$ implies $|X_s| = 0$, we can plug this bound into the result of Lemma~\ref{ap:lem:lsxxp} as follows:
\begin{equation*}
\norm{F_X - F_{X'}}_{2,\Gamma}
\leq
\frac{\Rmax}{1-\gamma}
\sqrt{\sum_{s \in \states_X} \frac{w_s}{|X_s|^2} +
\sum_{s \in \states \setminus \states_X} w_s}
=
\frac{\Rmax}{1-\gamma}
\sqrt{\sum_{s \in \states} \frac{w_s}{\max\{|X_s|,1\}^2}}
\enspace. \qedhere
\end{equation*}
\end{proof}

\begin{lemma}\label{ap:lem:maxkw}
The following holds for every $v \in \N^{\states}$:
\begin{equation*}
\varphi_k^w(v) = \sum_{s \in \states} \frac{w_s}{\max\{v_s - k,1\}^2} \enspace.
\end{equation*}
Furthermore, for every $k \geq \norm{v}_\infty - 1$ we have $\varphi_k^w(v) = \sum_s w_s$.
\end{lemma}
\begin{proof}
Recall that $\varphi_k^w(v) = \max_{\norm{v' - v}_\infty \leq k} \varphi^w(v')$ with $\varphi^w(v) = \sum_s w_s / \max\{v_s,1\}^2$ and observe the result follows immediately because
\begin{equation*}
\varphi_k^w(v) = \sum_{s \in \states} \frac{w_s}{\min_{-k \leq l \leq k} \max\{v_s + l, 1\}^2} =
\sum_{s \in \states} \frac{w_s}{\max\{v_s - k, 1\}^2} \enspace. \qedhere
\end{equation*}
\end{proof}

\section{Privacy Analysis of DP-LSL}\label{ap:sec:priv2}

\begin{lemma}
Let $X \simeq X'$ be two neighbouring datasets of $m$ trajectories with $X = (x_1,\ldots, x_{m-1},x)$ and $X' = (x_1,\ldots,x_{m-1},x')$. Let $F_x \in \R^{\states}$ (resp.\ $F_{x'} \in \R^{\states}$) be the vector given by $F_{x}(s) = F_{x,s}$ (resp.\ $F_{x'}(s) = F_{x',s}$). Define the diagonal matrices $\Gamma_\rho, \Delta_{x,x'} \in \R^{\states \times \states}$ given by $\Gamma_{\rho}(s,s) = \rho_s$ and $\Delta_{x,x'}(s,s) = \one_{s \in x} - \one_{s \in x'}$. If the regularization parameter satisfies $\lambda > \norm{\Phi^\top \Delta_{x,x'} \Gamma_{\rho} \Phi}$, then the following holds:
\begin{equation}\label{ap:eq:xxpl}
\frac{\norm{\theta_X^\lambda - \theta_{X'}^\lambda}_2}{2} \leq
\frac{\bignorm{\left(\Delta_{x,x'} \Phi \theta_X^\lambda - F_x + F_{x'} \right)^\top \Gamma_\rho \Phi}_2}{\lambda - \norm{\Phi^\top \Delta_{x,x'} \Gamma_{\rho} \Phi}}  \enspace.
\end{equation}
\end{lemma}
\begin{proof}
In order to simplify our notation we write $\bthe = \theta^\lambda_X$ and $\bthe' = \theta^\lambda_{X'}$ for the rest of the proof. Given a trajectory $x$ and a vector $\theta \in \R^d$ we shall also write $\ell(x, \theta) = \sum_{s \in \states_x} \rho_s (F_{x,s} - \phi_s^\top \theta)^2$ so that $J_X(\theta) = \frac{1}{m} \sum_{i = 1}^m \ell(x_i, \theta)$. Now we proceed with the proof.

Let us start by noting that because $J^\lambda_X(\theta)$ is $\lambda/m$-strongly convex, we have $J^\lambda_X(\theta_1) - J^\lambda_X(\theta_2) \geq \langle \nabla J^\lambda_X(\theta_2), \theta_1 - \theta_2 \rangle + \frac{\lambda}{2 m} \norm{\theta_1 - \theta_2}_2^2$ for any $\theta_1, \theta_2 \in \R^d$. Thus, using that optimality implies $\nabla J^\lambda_X(\bthe) = \nabla J^\lambda_{X'}(\bthe') = 0$, we get
\begin{align*}
\frac{\lambda}{m} \norm{\bthe - \bthe'}_2^2 &\leq
J^\lambda_X(\bthe') - J^\lambda_X(\bthe) + J^\lambda_{X'}(\bthe) - J^\lambda_{X'}(\bthe') \\
&= J_{X}(\bthe') - J_X(\bthe) + J_{X'}(\bthe) - J_{X'}(\bthe') \\
&= \frac{1}{m} \left( \ell(x, \bthe') - \ell(x, \bthe) + \ell(x', \bthe) - \ell(x', \bthe') \right)
\enspace,
\end{align*}
where the equalities follows from definitions of $X$, $X'$, $J^\lambda_X$ and $J_X$. If we now expand the definition of $\ell(x,\theta)$ we see that
\begin{align*}
\ell(x, \bthe') - \ell(x, \bthe) &=
\sum_{s \in \states_x} \rho_s \left( (\phi_s^\top \bthe')^2 - (\phi_s^\top \bthe)^2 - 2 F_{x,s} \phi_s^\top (\bthe' - \bthe) \right) \enspace, \\
\ell(x', \bthe) - \ell(x', \bthe') &=
\sum_{s \in \states_{x'}} \rho_s \left( (\phi_s^\top \bthe)^2 - (\phi_s^\top \bthe')^2 - 2 F_{{x'},s} \phi_s^\top (\bthe - \bthe') \right) \enspace.
\end{align*}
Using the identity $(\phi_s^\top \bthe')^2 - (\phi_s^\top \bthe)^2 = (\bthe' + \bthe)^\top \phi_s \phi_s^\top (\bthe' - \bthe)$, we rewrite $\ell(x, \bthe') - \ell(x, \bthe) + \ell(x', \bthe) - \ell(x', \bthe')$ as
\begin{equation}
\sum_{s \in \states} \rho_s \left[ (\one_{s \in x} - \one_{s \in x'}) (\bthe' + \bthe)^\top \phi_s \phi_s^\top
- 2 (F_{x,s} - F_{x',s}) \phi_s^\top \right] (\bthe' - \bthe) \enspace,
\end{equation}
where we implicitly used that $F_{x,s} = 0$ whenever $s \notin x$. Finally, using the definitions in the statement we can rearrange the above expression to show that
\begin{align*}
\frac{\lambda}{m} \norm{\bthe - \bthe'}_2^2 &\leq
\frac{1}{m} \left( (\bthe' + \bthe)^\top \Phi^\top \Delta_{x,x'} - 2 (F_x - F_{x'})^\top \right) \Gamma_\rho \Phi (\bthe' - \bthe) \\
&= \frac{2}{m}
\left( \bthe^\top \Phi^\top \Delta_{x,x'} - (F_x - F_{x'})^\top \right) \Gamma_\rho \Phi (\bthe' - \bthe) 
+ \frac{1}{m} (\bthe' - \bthe)^\top \Phi^\top \Delta_{x,x'} \Gamma_\rho \Phi (\bthe' - \bthe) \\
&\leq \frac{2}{m} \norm{\left( \bthe^\top \Phi^\top \Delta_{x,x'} - (F_x - F_{x'})^\top \right) \Gamma_\rho \Phi}_2 \norm{\bthe' - \bthe}_2
+ \frac{1}{m} \norm{\Phi^\top \Delta_{x,x'} \Gamma_\rho \Phi} \norm{\bthe' - \bthe}_2^2 \enspace,
\end{align*}
where we used the Cauchy--Schwartz inequality and the definition of operator norm. The result now follows by solving for $\norm{\bthe - \bthe'}_2$ in the above inequality.
\end{proof}

\begin{corollary}\label{ap:cor:LSl}
Let $X$ be a dataset of trajectories and suppose $\lambda > \norm{\Phi}^2 \norm{\rho}_\infty$. Then the following holds for any neighbouring dataset $X' \simeq X$:
\begin{equation*}
\norm{\theta_X^\lambda - \theta_{X'}^\lambda}_2 \leq
\frac{2 \Rmax \norm{\Phi}}{(1-\gamma) (\lambda - \norm{\Phi}^2 \norm{\rho}_\infty)} \sqrt{\varphi_X^\lambda} \enspace,
\end{equation*}
where
\begin{equation*}
\varphi_X^\lambda = \left(\frac{\norm{\Phi} \norm{\rho}_\infty}{\sqrt{2 \lambda}} \sqrt{\sum_{s \in \states} \rho_s |X_s|} + \norm{\rho}_2\right)^2 \enspace.
\end{equation*}
\end{corollary}
\begin{proof}
We start by noting that $\norm{\Delta_{x,x'}} \leq 1$ and $\norm{\Gamma_\rho} = \norm{\rho}_\infty$, hence submultiplicativity of matrix operator norms yields $\norm{\Phi^\top \Delta_{x,x'} \Gamma_\rho \Phi} \leq \norm{\Phi}^2 \norm{\rho}_\infty$. On the other hand, for the numerator in \eqref{ap:eq:xxpl} we have
\begin{equation}\label{ap:eq:numerator}
\bignorm{\left(\Delta_{x,x'} \Phi \theta_X^\lambda - F_x + F_{x'} \right)^\top \Gamma_\rho \Phi}_2 \leq
\left(\norm{\theta_X^\lambda}_2 \norm{\Phi} \norm{\rho}_\infty + \norm{(F_x - F_{x'})^\top \Gamma_\rho}_2 \right) \norm{\Phi} \enspace.
\end{equation}
Bounding the individual entries in $F_x$ and $F_{x'}$ by $\Rmax / (1-\gamma)$ we get $\norm{(F_x - F_{x'})^\top \Gamma_\rho}_2 \leq \Rmax \norm{\rho}_2 / (1-\gamma)$. The last step is to bound the norm $\norm{\theta^\lambda_X}_2$, for which we use the closed-form solution to $\argmin_\theta J_X^\lambda(\theta)$ given in the paper and write:
\begin{equation*}
\norm{\theta^\lambda_X}_2 \leq \bignorm{(\Phi^\top \Gamma_X \Phi + \frac{\lambda}{2m} I)^{-1} \Phi^\top \Gamma_X^{1/2}} \norm{F_X}_{2,\Gamma_X} \leq \bignorm{(\Phi^\top \Gamma_X \Phi + \frac{\lambda}{2m} I)^{-1} \Phi^\top \Gamma_X^{1/2}} \left(\frac{\Rmax}{1 - \gamma} \sqrt{\sum_{s \in \states} \frac{\rho_s |X_s|}{m}} \right)
\enspace.
\end{equation*}
To bound the last remaining norm let use write $U \Sigma V^\top$ for the SVD of $\Gamma_X^{1/2} \Phi$, where $V \in \R^{d \times d}$ with $V^\top V = V V^\top = I$. With this we can write:
\begin{equation}
(\Phi^\top \Gamma_X \Phi + \frac{\lambda}{2m} I)^{-1} \Phi^\top \Gamma_X^{1/2} =
V \left(\Sigma^2 + \frac{\lambda}{2m} I\right)^{-1} \Sigma U^\top \enspace.
\end{equation}
Now we use that $\norm{U} = \norm{V} = 1$ and $x / (x^2 + a) \leq 1/(2 \sqrt{a})$ for any $x \geq 0$ to get
$\norm{V (\Sigma^2 + (\lambda / 2m) I)^{-1} \Sigma U^\top} \leq \sqrt{m / 2 \lambda}$. Thus we get a bound for $\norm{\theta_X^\lambda}_2$ that when plugged into \eqref{ap:eq:numerator} yields the desired result.
\end{proof}

\begin{lemma}\label{ap:lem:phikl}
The following holds for every $v \in \N^{\states}$:
\begin{equation*}
\varphi_k^\lambda(v) = \left(\frac{\norm{\Phi} \norm{\rho}_\infty}{\sqrt{2 \lambda}} \sqrt{\sum_{s \in \states} \rho_s \max\{v_s + k, m\}} + \norm{\rho}_2\right)^2 \enspace.
\end{equation*}
Furthermore, for every $k \geq m - \min_s v_s$ we have $\varphi_k^\lambda(v) = \left(\frac{\norm{\Phi} \norm{\rho}_\infty \sqrt{m}}{\sqrt{2 \lambda}} \sqrt{\sum_{s \in \states} \rho_s} + \norm{\rho}_2\right)^2$.
\end{lemma}
\begin{proof}
The proof is similar to that of Lemma~\ref{ap:lem:maxkw} and is omitted.
\end{proof}

\section{Utility Analysis of DP-LSW}\label{ap:sec:util1}

The goal of this section is to show that as the size $m$ of the dataset $X$ grows, the differentially private solution $\theta_X^w$ provided by algorithm DP-LSW is not much worse than the one obtained by directly minimizing $J_X^w(\theta)$. In other words, for large datasets the noise introduced by the privacy constraint is negligible. We do so by proving a $O(1/m^2)$ bound for the expected empirical excess risk given by $\E_{X,\eta} [J_X^w(\hat{\theta}_X^w) - J_X^w(\theta_X^w)]$. Our analysis starts with a lemma that leverages the law of total expectation in order to reduce the bound to a quantity that only depends on $\E_X[\sigma_X^2]$.

\begin{lemma}\label{ap:lem:totalexpw}
\begin{equation}
\E_{X,\eta} [J_X^w(\hat{\theta}_X^w) - J_X^w(\theta_X^w)] =
\norm{\Gamma^{1/2} \Phi}_F^2
%\Tr(\Phi^\top \Gamma \Phi)
\E_X [\sigma_X^2] \enspace.
\end{equation}
\end{lemma}
\begin{proof}
By the law of total expectation it is enough to show that
\begin{equation}
\E_\eta [J_X^w(\hat{\theta}_X^w) - J_X^w(\theta_X^w) | X] = \sigma_X^2
%\Tr(\Phi^\top \Gamma \Phi)
\norm{\Gamma^{1/2} \Phi}_F^2
\enspace.
\end{equation}
Let $X$ be an arbitrary dataset. Expanding the definition of $J_X^w(\theta)$ we have that for any $\theta \in \R^d$
\begin{equation}
J_X^w(\theta) = F_X^\top \Gamma F_X + \theta^\top \Phi^\top \Gamma \Phi \theta - 2 F_X^\top \Gamma \Phi \theta \enspace.
\end{equation}
On the other hand, since $\nabla_\theta J_X^w(\theta_X^w) = 0$, we have ${\theta_X^w}^\top \Phi^\top \Gamma \Phi = F_X^\top \Gamma \Phi$. Thus, using the definition $\hat{\theta}_X^w = \theta_X^w + \eta$, a simple algebraic calculation yields
\begin{equation}
J_X^w(\hat{\theta}_X^w) - J_X^w(\theta_X^w) = \eta^\top \Phi^\top \Gamma \Phi \eta - F_X^\top \Gamma \Phi \eta - \eta^\top \Phi^\top \Gamma \Phi \theta_X^w \enspace.
\end{equation}
Finally, taking the expectation over $\eta \sim \cN(0,\sigma_X^2 I)$ of the above expression we get
\begin{equation}
\E_\eta [J_X^w(\hat{\theta}_X^w) - J_X^w(\theta_X^w)] = \E_\eta[ \eta^\top \Phi^\top \Gamma \Phi \eta] = \sigma_X^2 \Tr(\Phi^\top \Gamma \Phi)
= \sigma_X^2 \norm{\Gamma^{1/2} \Phi}_F^2
\enspace. \qedhere
\end{equation}
\end{proof}

In order to bound $\E_X[\sigma_X^2]$ we recall the variance has the form $\sigma_X^2 = C^2 \psi_X^w$, where $C$ is a constant independent of $X$ and
\begin{equation}\label{ap:eq:psixleqmax}
\psi_X^w = \max_{k \geq 0} e^{-k \beta} \sum_{s \in \states} \frac{w_s}{\max\{|X_s| - k, 1\}^2}  \leq \sum_{s} w_s \left(\max_{k \geq 0} \frac{e^{-k \beta}}{\max\{|X_s| - k, 1\}^2}\right) \enspace.
\end{equation}
Thus, we can bound $\E_X[\sigma_X^2] = C^2 \E_X [\psi_X^w]$ by providing a bound for the expectation of each individual maximum in \eqref{ap:eq:psixleqmax}. The two following technical lemmas will prove useful.

\begin{lemma}\label{ap:lem:max}
Let $b > 0$ and $a \geq 1$. Then the following holds:
\begin{equation}
\max_{0 \leq x \leq a - 1} \frac{e^{-b x}}{(a - x)^2} =
\begin{cases}
\frac{1}{a^2} & b < 2/a \\
e^{1 - a b} & b > 2 \\
\frac{e^2}{4} b^2 e^{-a b} & \text{otherwise}
\end{cases}
\end{equation}
\end{lemma}
\begin{proof}
The result follows from a simple calculation.
\end{proof}

\begin{lemma}\label{ap:lem:bin}
Suppose $B_{m,p}$ is a binomial random variable with $m$ trials and success probability $p$. Then the following hold:
\begin{align*}
\E\left[\frac{1}{B_{m,p} + 1}\right] &= \frac{1- (1-p)^{m+1}}{p (m+1)} \enspace,\\
\E\left[\frac{1}{B_{m,p}^2} \one_{B_{m,p} \geq 1}\right] &\leq
\frac{6}{p (m+1)} \left( \frac{1 - (1-p)^{m+2}}{p (m+2)} - (1-p)^{m+1} - \frac{p (m+1)}{2} (1-p)^m \right) \enspace.
\end{align*}
\end{lemma}
\begin{proof}
The first expectation is a classical exercise in probability textbooks. The second one can be proved as follows:
\begin{align*}
\E\left[\frac{1}{B_{m,p}^2} \one_{B_{m,p} \geq 1}\right] &=
\sum_{k = 1}^m \frac{1}{k^2} \binom{m}{k} p^k (1-p)^{m-k} \\ &\leq
6 \sum_{k=1}^m \frac{1}{(k+1)(k+2)} \binom{m}{k} p^k (1-p)^{m-k} \\ &=
\frac{6}{p (m+1)} \sum_{k=1}^m \frac{1}{k+2} \frac{(m+1)!}{(k+1)! (m-k)!} p^{k+1} (1-p)^{m-k} \\ &=
\frac{6}{p (m+1)} \sum_{k=1}^m \frac{1}{k+2} \P[B_{m+1,p} = k+1] \\ &=
\frac{6}{p (m+1)} \sum_{j=2}^{m+1} \frac{1}{j+1} \P[B_{m+1,p} = j] \\ &=
\frac{6}{p (m+1)} \left( \E\left[\frac{1}{B_{m+1,p} + 1}\right]  - \P[B_{m+1,p} = 0] - \frac{1}{2} \P[B_{m+1,p} = 1] \right) \\ &=
\frac{6}{p (m+1)} \left( \frac{1 - (1-p)^{m+2}}{p (m+2)} - (1-p)^{m+1} - \frac{p (m+1)}{2} (1-p)^m \right) \enspace,
\end{align*}
where we used the first equation in the last step, and the bound $(k+1)(k+2)/k^2 \leq 6$ for $k \geq 1$ in the first inequality.
\end{proof}

Recall that $p_s$ denotes the probability that a trajectory from $X$ visits states $s$. Because these trajectories are i.i.d.\ we have that $|X_s| = B_{m,p_s}$ is a binomial random variable. Therefore, we can combine the last two lemmas to prove the following.

\begin{lemma}\label{ap:lem:expmax}
Suppose $\beta  \leq 2$. Then we have:
\begin{equation}
\E_X\left[\max_{k \geq 0} \frac{e^{-k \beta}}{\max\{|X_s| - k, 1\}^2}\right] \leq
\begin{cases}
\frac{6}{p_s^2 (m+1)(m+2)} + \frac{e^2 \beta^2}{4} (1 - (1- e^{-\beta})p_s)^m & p_s > 0 \enspace, \\
1 & p_s = 0 \enspace.
\end{cases}
\end{equation}
\end{lemma}
\begin{proof}
Note in the first place that Lemma~\ref{ap:lem:max} implies
\begin{equation}\label{ap:eq:maxindicators}
\max_{k \geq 0} \frac{e^{-k \beta}}{\max\{|X_s| - k, 1\}^2} =
\one_{|X_s| = 0} + \one_{1 \leq |X_s| < 2 / \beta} \frac{1}{|X_s|^2} + \one_{|X_s| \geq 2 / \beta} \frac{e^2}{4} \beta^2 e^{-\beta |X_s|} \enspace,
\end{equation}
where we used that in the case $|X_s| = 0$ the maximum is $1$. If $p_s = 0$, then obviously $|X_s| = 0$ almost surely and the expectation of \eqref{ap:eq:maxindicators} equals $1$. On the other hand, when $p_s > 0$ we use the linearity of expectation and bound each term separately. Clearly, $\E_X[\one_{|X_s| = 0}] = \P_X [B_{m,p_s} = 0] = (1-p_s)^m$. On the other hand, by looking up the moment generating function of a binomial distribution we have
\begin{equation}
\E_X[\one_{|X_s| \geq 2 / \beta} \frac{e^2}{4} \beta^2 e^{-\beta |X_s|}] \leq
\frac{e^2}{4} \beta^2 \E_X[e^{-\beta |X_s|}]
= \frac{e^2}{4} \beta^2 (1 - (1- e^{-\beta}) p_s )^m \enspace.
\end{equation}
The remaining term is bounded by
\begin{equation}
\E_X\left[\one_{1 \leq |X_s| < 2 / \beta} \frac{1}{|X_s|^2}\right] \leq \E_X\left[\one_{1 \leq |X_s|} \frac{1}{|X_s|^2}\right] \enspace.
\end{equation}
Therefore, applying Lemma~\ref{ap:lem:bin} and upper bounding some negative terms by zero, we get
\begin{equation}
\E_X\left[\max_{k \geq 0} \frac{e^{-k \beta}}{\max\{|X_s| - k, 1\}^2}\right] \leq
\frac{6}{p_s^2 (m+1)(m+2)} + \frac{e^2 \beta^2}{4} (1 - (1- e^{-\beta})p_s)^m \enspace.
\end{equation}
\end{proof}

Now we can combine Lemmas~\ref{ap:lem:totalexpw} and~\ref{ap:lem:expmax} using Equation~\ref{ap:eq:psixleqmax} to get our final result.

\begin{theorem}
Let $\states_0 = \{ s \in \states | p_s = 0 \}$ and $S_+ = \states \setminus \states_0$. Let
$C = \alpha \Rmax \norm{(\Gamma^{1/2} \Phi)^{\dagger}}
%\Tr(\Phi^\top \Gamma \Phi)
\norm{\Gamma^{1/2} \Phi}_F
/ (1-\gamma)$. Suppose $\beta \leq 2$. Then we have the following:
\begin{equation*}
\E_{X,\eta}[J_X^w(\hat{\theta}_X^w) - J_X^w(\theta_X^w)] \leq
C^2 \left( \sum_{s \in \states_0} w_s + \sum_{s \in \states_+} w_s \left(\frac{6}{p_s^2 (m+1)(m+2)} + \frac{e^2 \beta^2}{4} (1 - (1- e^{-\beta})p_s)^m \right)\right) \enspace.
\end{equation*}
\end{theorem}

The following version is the one given in the paper for reasons of space. It is easily obtained by noting that $e^2 / 4 \leq 6$, $m^2 \leq (m+1)(m+2)$, and when $\beta \leq 1/2$ then $1- (1-e^{-\beta}) p_s \leq 1 - \beta p_s / 2$.

\begin{corollary}
Let $\states_0 = \{ s \in \states | p_s = 0 \}$ and $S_+ = \states \setminus \states_0$. Let
$C = \alpha \Rmax \norm{(\Gamma^{1/2} \Phi)^{\dagger}}
\norm{\Gamma^{1/2} \Phi}_F
/ (1-\gamma)$. Suppose $\beta \leq 1/2$. Then $\E_{X,\eta}[J_X^w(\hat{\theta}_X^w) - J_X^w(\theta_X^w)] $ is upper bounded by:
\begin{equation*}
C^2 \left( \sum_{s \in \states_0} w_s + 6 \sum_{s \in \states_+} w_s \left(\frac{1}{p_s^2 m^2} + \beta^2 \left(1 - \frac{\beta p_s}{2}\right)^m \right) \right) \enspace.
\end{equation*}
\end{corollary}

The following is an immediate consequence of these results.

\begin{corollary}
If $w_s = 0$ for all $s \in \states_0$, then $\E_{X,\eta}[J_X^w(\hat{\theta}_X^w) - J_X^w(\theta_X^w)] = O(1/m^2)$.
\end{corollary}

\section{Utility Analysis of DP-LSL}\label{ap:sec:util2}

The analysis in this section follows a scheme similar to the previous one. We start by taking the expectation of the excess empirical risk with respect to the Gaussian perturbation $\eta$.

\begin{lemma}\label{ap:lem:totalexpl}
\begin{equation}\label{ap:eq:Exetal}
\E_{X,\eta} [J_X^\lambda(\hat{\theta}_X^\lambda) - J_X^\lambda(\theta_X^\lambda)] = \E_X \left[\left(\frac{\lambda d}{2 m} + \frac{1}{m} \sum_{s \in \states} \rho_s \norm{\phi_s}_2^2 |X_s| \right) \sigma_X^2 \right] \enspace.
\end{equation}
\end{lemma}
\begin{proof}
Let $X$ be an arbitrary dataset with $m$ trajectories. Recalling that $\hat{\theta}_X^\lambda = \theta_X^\lambda + \eta$ we get:
\begin{align*}
J_X^\lambda(\hat{\theta}_X^\lambda)
-
J_X^\lambda(\theta_X^\lambda)
&=
\frac{1}{m}
\sum_{i = 1}^m \sum_{s \in \states_{x_i}}
\rho_s
\left(
(\phi_s^\top \hat{\theta}_X^\lambda)^2
-
(\phi_s^\top \theta_X^\lambda)^2
-
2 F_{x_i,s} \phi_s^\top \eta
\right)
+ \frac{\lambda}{2m}
\left(
\norm{\hat{\theta}_X^\lambda}_2^2
-
\norm{\theta_X^\lambda}_2^2
\right)
\\
&=
\frac{1}{m}
\sum_{i = 1}^m \sum_{s \in \states_{x_i}}
\rho_s
\left(
\eta^\top \phi_s \phi_s^\top \eta
+
2 \eta^\top \phi_s \phi_s^\top \theta_X^\lambda
-
2 F_{x_i,s} \phi_s^\top \eta
\right)
+ \frac{\lambda}{2m}
\left(
\norm{\eta}_2^2
+
2 \eta^\top \theta_X^\lambda
\right) \enspace.
\end{align*}
Taking the expectation over $\eta \sim \cN(0,\sigma_X^2 I)$ in the above expression we get
\begin{align*}
\E_\eta[J_X^\lambda(\hat{\theta}_X^\lambda) - J_X^\lambda(\theta_X^\lambda)] &=
\frac{1}{m} \sum_{i = 1}^m \sum_{s \in \states_{x_i}} \rho_s \Tr(\phi_s \phi_s^\top) \sigma_X^2
+ \frac{\lambda}{2 m} d \sigma_X^2 \enspace.
\end{align*}
The result now follows from noting that $\sum_{i = 1}^m \sum_{s \in \states_{x_i}} \rho_s \Tr(\phi_s \phi_s^\top) = \sum_{s \in \states} \rho_s \norm{\phi_s}_2^2 |X_s|$.
\end{proof}

In order to bound the expression given by previous lemma we will expand the definition of $\sigma_X = C_\lambda \sqrt{\psi_X^\lambda}$, with $C_\lambda = 2 \Rmax \norm{\Phi} / (1-\gamma) (\lambda - \norm{\Phi}^2 \norm{\rho}_\infty)$, and note that using the straightforward bound $(a+b)^2 \leq 2a^2 + 2b^2$ we have:
\begin{align*}
\psi_X^\lambda &= \max_{k \geq 0} e^{-k \beta} \left(
\norm{\rho}_2 + \frac{\norm{\Phi} \norm{\rho}_\infty}{\sqrt{2 \lambda}} \sqrt{\sum_{s \in \states} \rho_s \min\{|X_s| + k ,m\}}
\right)^2 \\
&\leq 2 \norm{\rho}_2^2 + \frac{\norm{\Phi}^2 \norm{\rho}_\infty^2}{\lambda} \sum_{s \in \states} \rho_s \max_{k \geq 0} \left(e^{-2 k \beta} \min\{|X_s| + k ,m \} \right) \enspace.
\end{align*}
The following lemma can be used to bound the maximums inside this sum.

\begin{lemma}
Suppose $a \geq 0$ and $b > 0$. Then the following holds:
\begin{equation}
\max_{0 \leq x \leq m - a} e^{-2 b x}(a + x) =
\begin{cases}
a & b < a/2 \\
m e^{-2 b (m-a)} & b > m/2 \\
\frac{1}{2 e b} e^{2 a b} & \text{otherwise}
\end{cases}
\end{equation}
\end{lemma}

Assuming we have $2 \beta < 1 \leq m$, previous lemma yields:
\begin{equation}\label{ap:eq:ubmaxl}
\max_{k \geq 0} \left(e^{- 2 k \beta} \min\{ |X_s| + k, m \}\right) = |X_s| \one_{|X_s| > 2 \beta} + \frac{1}{2 e \beta} e^{2 \beta |X_s|} \one_{|X_s| \leq 2 \beta} \leq |X_s| + \frac{1}{2 e \beta} \one_{|X_s| = 0}
\enspace.
\end{equation}

When taking the expectation of the upper bound for \eqref{ap:eq:Exetal} obtained by plugging in \eqref{ap:eq:ubmaxl}, several quantities involving products of correlated binomial random variables will appear. Next lemma gives expressions for all these expectations.

\begin{lemma}\label{ap:lem:pssp}
Recall that $p_s = \P[s \in x]$ and $|X_s|$ is a binomial random variable with $m$ trials and success probability $p_s$. Define $p_{s,s'} = \P[s \in x \wedge s' \in x]$ and $\bar{p}_{s,s'} = \P[s \in x \wedge s' \notin x]$ for any $s, s' \in \states$. Then we have the following:
\begin{enumerate}
\item $\E[|X_s|] = m p_s$,
\item $\E[\one_{|X_s| = 0}] = (1-p_s)^m$,
\item $\E[|X_s|^2] = m^2 p_s^2 + m (p_s - p_s^2)$,
\item $\E[|X_s| |X_{s'}|] = m(m-1) p_s p_{s'} + m p_{s,s'}$,
\item $\E[|X_s| \one_{|X_{s'}|=0}] = m \bar{p}_{s,s'} (1 - p_{s'})^{m-1}$.
\end{enumerate}
\end{lemma}
\begin{proof}
All equations follow from straightforward calculations.
\end{proof}

\begin{theorem}
Suppose $\beta < 1/2$ and $\lambda > \norm{\Phi}^2 \norm{\rho}_\infty$. Let $C_\lambda = 2 \alpha \Rmax \norm{\Phi}/(1 - \gamma) (\lambda - \norm{\Phi}^2 \norm{\rho}_\infty)$. Then we have
\begin{align*}
\E_{X,\eta}[ & J_X^\lambda(\hat{\theta}_X^\lambda) - J_X^\lambda(\theta_X^\lambda) ] \leq C_\lambda^2
\left\{
\sum_{s \in \states} \rho_s p_s \left(\frac{d \norm{\Phi}^2 \norm{\rho}_\infty^2}{2} + 2 \norm{\rho}_2^2 \norm{\phi_s}_2^2 \right)
\right. \\
&+
\frac{\lambda}{m} d \norm{\rho}_2^2
+
\frac{1}{m} \frac{d \norm{\Phi}^2 \norm{\rho}_\infty^2}{4 e \beta} \sum_{s \in \states} \rho_s (1-p_s)^m
+
\frac{m}{\lambda} \norm{\Phi}^2 \norm{\rho}_\infty^2 \sum_{s,s' \in \states} \rho_s \rho_{s'} p_s p_{s'} \norm{\phi_s}_2^2 \\
&+
\left.
\frac{1}{\lambda} \norm{\Phi}^2 \norm{\rho}_\infty^2 \left(
\sum_{s \in \states} \rho_s^2 \norm{\phi_s}_2^2 (p_s - p_s^2)
+
\sum_{\substack{s, s' \in \states \\ s \neq s'}} \rho_s \rho_{s'} \norm{\phi_s}_2^2 \left(p_{s,s'} - p_s p_{s'} + \frac{1}{2 e \beta} \bar{p}_{s,s'} (1 - p_{s'})^{m-1}\right)\right) \right\} \enspace.
\end{align*}
\end{theorem}
\begin{proof}
Combining Lemma~\ref{ap:lem:totalexpl} with \eqref{ap:eq:ubmaxl} and the definition of $\sigma_X^2$ yields the following upper bound for $\E_{X,\eta}[J_X^\lambda(\hat{\theta}_X^\lambda) - J_X^\lambda(\theta_X^\lambda) ]$:
\begin{equation*}
C_\lambda^2
\E_X \left[
\left(\frac{\lambda d}{2 m} + \frac{1}{m} \sum_{s \in \states} \rho_s \norm{\phi_s}_2^2 |X_s| \right)
\left(
2 \norm{\rho}_2^2 + \frac{\norm{\Phi}^2 \norm{\rho}_\infty^2}{\lambda} \sum_{s \in \states} \rho_s 
\left(|X_s| + \frac{1}{2 e \beta} \one_{|X_s| = 0}\right) \right) \right] \enspace.
\end{equation*}
Terms that do not involve products of the form $|X_s| |X_{s'}|$ or $|X_s| \one_{|X_{s'}| = 0}$ can be straightforwardly reduced to linear combinations of expectations in Lemma~\ref{ap:lem:pssp}. The remaining term yields the following: 
\begin{align*}
\E_X &\left[ \sum_{s,s' \in \states} \rho_s \rho_{s'} \norm{\phi_s}_2^2 |X_s| \left(|X_{s'}| + \frac{1}{2 e \beta} \one_{|X_{s'}| = 0}\right) 	\right] \\
&=
\sum_{s \in \states} \rho_s^2 \norm{\phi_s}_2^2 \E_X \left[ |X_s| \left(|X_s| +  \frac{1}{2 e \beta} \one_{|X_{s}| = 0} \right)\right] \\
&\enspace +
\sum_{\substack{s, s' \in \states \\ s \neq s'}} \rho_s \rho_{s'} \norm{\phi_s}_2^2 \E_X
\left[ |X_s| \left(|X_{s'}| + \frac{1}{2 e \beta} \one_{|X_{s'}| = 0}\right) 	\right] \\
&= \sum_{s \in \states} \rho_s^2 \norm{\phi_s}_2^2 \left(m^2 p_s^2 + m (p_s - p_s^2)\right) \\
&\enspace +
\sum_{\substack{s, s' \in \states \\ s \neq s'}} \rho_s \rho_{s'} \norm{\phi_s}_2^2 \left(m (m-1) p_s p_{s'} + m p_{s,s'} + \frac{1}{2 e \beta} m \bar{p}_{s,s'} (1 - p_{s'})^{m-1}\right) \enspace,
\end{align*}
where we used Lemma~\ref{ap:lem:pssp} again. Thus we get:
\begin{align*}
\E_{X,\eta}[J_X^\lambda(\hat{\theta}_X^\lambda) - J_X^\lambda(\theta_X^\lambda) ] & \leq C_\lambda^2 
\left\{
\frac{\lambda d \norm{\rho}_2^2}{m}
+
\frac{d \norm{\Phi}^2 \norm{\rho}_\infty^2}{2 m} \sum_{s \in \states} \rho_s \left(m p_s + \frac{1}{2 e \beta} (1-p_s)^m\right)
\right. \\
&+
\frac{2 \norm{\rho}_2^2}{m} \sum_{s \in \states} \rho_s p_s \norm{\phi_s}_2^2 m +
\frac{\norm{\Phi}^2 \norm{\rho}_\infty^2}{\lambda m} \sum_{s \in \states} \rho_s^2 \norm{\phi_s}_2^2 \left(m^2 p_s^2 + m (p_s - p_s^2)\right)
\\
&+
\left.
\frac{\norm{\Phi}^2 \norm{\rho}_\infty^2}{\lambda m} \sum_{\substack{s, s' \in \states \\ s \neq s'}} \rho_s \rho_{s'} \norm{\phi_s}_2^2 \left(m (m-1) p_s p_{s'} + m p_{s,s'} + \frac{1}{2 e \beta} m \bar{p}_{s,s'} (1 - p_{s'})^{m-1}\right)
\right\}
\end{align*}
The final result is obtained by grouping the terms in this expression by their dependence in $\lambda$ and $m$.
\end{proof}

Note that if we take $\lambda = \omega(1)$ with respect to $m$ in the above theorem, then $C_\lambda = O(1/\lambda)$ and we get the following corollary.

\begin{corollary}
Suppose $\lambda = \omega(1)$ with respect to $m$. Then we have
\begin{equation}
\E_{X,\eta}[J_X^\lambda(\hat{\theta}_X^\lambda) - J_X^\lambda(\theta_X^\lambda) ] = O\left(\frac{1}{\lambda m} + \frac{1}{\lambda^2} + \frac{m}{\lambda^3} \right) \enspace.
\end{equation}
\end{corollary}

\end{document}